\documentclass[sigconf]{acmart}

\AtBeginDocument{%
  }

\copyrightyear{2024}
\acmYear{2024}
\setcopyright{acmlicensed}\acmConference[KDD '24]{Proceedings of the 30th
ACM SIGKDD Conference on Knowledge Discovery and Data Mining}{August
25--29, 2024}{Barcelona, Spain}
\acmBooktitle{Proceedings of the 30th ACM SIGKDD Conference on Knowledge
Discovery and Data Mining (KDD '24), August 25--29, 2024, Barcelona, Spain}
\acmDOI{10.1145/3637528.3671774}
\acmISBN{979-8-4007-0490-1/24/08}

\usepackage[utf8]{inputenc} 
\usepackage[T1]{fontenc}    
\usepackage{hyperref}       
\usepackage{url}            
\usepackage{booktabs}       
\usepackage{amsfonts}       
\usepackage{nicefrac}       
\usepackage{microtype}      
\usepackage{url}
\usepackage{xcolor}         
\usepackage{bm}
\usepackage{xspace}
\usepackage{multirow} 
\usepackage{graphicx}
\usepackage{subfig}
\usepackage{makecell}
\usepackage[ruled,vlined]{algorithm2e}
\usepackage{amsthm}
\usepackage{amsmath}
\theoremstyle{plain}
\newtheorem{theorem}{Theorem}

\usepackage{balance} 
\usepackage{pifont}
\usepackage{makecell}
\theoremstyle{plain}
\newtheorem*{theorem*}{Theorem}

\theoremstyle{remark}

\newcommand{\ours}{$R^{2}$LP\xspace}
\begin{document}

\title{Resurrecting Label Propagation for Graphs with Heterophily and Label Noise}

\author{Yao Cheng}
\affiliation{%
  \institution{East China Normal University}
  \city{Shanghai}
  \country{China}}
\email{yaocheng_623@stu.ecnu.edu.cn}

\author{Caihua Shan}
\affiliation{%
  \institution{Microsoft Research Asia}
  \city{Shanghai}
  \country{China}}
\email{caihuashan@microsoft.com}

\author{Yifei Shen}
\affiliation{%
  \institution{Microsoft Research Asia}
  \city{Shanghai}
  \country{China}}
\email{yshenaw@connect.ust.hk}

\author{Xiang Li
}
\authornote{Xiang Li is the corresponding author}
\affiliation{%
  \institution{East China Normal University}
  \city{Shanghai}
  \country{China}}
\email{xiangli@dase.ecnu.edu.cn}

\author{Siqiang Luo}
\affiliation{%
  \institution{Nanyang Technological University}
  \city{Singapore}
  \country{Singapore}
  }
\email{siqiang.luo@ntu.edu.sg}

\author{Dongsheng Li}
\affiliation{%
  \institution{Microsoft Research Asia}
  \city{Shanghai}
  \country{China}}
\email{dongsheng.li@microsoft.com
}
\renewcommand{\shortauthors}{Cheng et al.}
\acmArticleType{Research}

\begin{abstract}
  Label noise is a common challenge in large datasets, as it can significantly degrade the generalization ability of deep neural networks. Most existing studies focus on noisy labels in computer vision; however, graph models encompass both node features and graph topology as input, and become more susceptible to label noise through message-passing mechanisms.
  Recently, only a few works have been proposed to tackle the label noise on graphs. 
  One significant limitation is that they operate under the assumption that the graph exhibits homophily and that the labels are distributed smoothly.
  However, real-world graphs can exhibit varying degrees of heterophily, or even be dominated by heterophily, which results in the inadequacy of the current methods.
  
  In this paper, we study graph label noise in the context of arbitrary heterophily, with the aim of rectifying noisy labels and assigning labels to previously unlabeled nodes.
  We begin by conducting two empirical analyses to explore the impact of graph homophily on graph label noise. Following observations, we propose a 
  efficient algorithm, denoted as \ours.
  Specifically, \ours\ is an iterative algorithm with three steps: (1) reconstruct the graph to recover the homophily property, (2) utilize label propagation to rectify the noisy labels, (3) select high-confidence labels to retain for the next iteration. By iterating these steps, we obtain a set of ``correct'' labels, ultimately achieving high accuracy in the node classification task. The theoretical analysis is also provided to demonstrate its remarkable denoising effect.
  Finally, we perform experiments on ten benchmark datasets with different levels of graph heterophily and various types of noise. In these experiments, we compare the performance of \ours\ against ten typical baseline methods.
  Our results illustrate the superior performance of the proposed \ours. 
  The code and data of this paper can be accessed at: \url{https://github.com/cy623/R2LP.git}.

\end{abstract}



\maketitle

\section{Introduction}
Graphs are ubiquitous in the real world, 
such as social networks~\cite{leskovec2010predicting}, 
biomolecular structures~\cite{zitnik2017predicting} and knowledge graphs~\cite{hogan2021knowledge}. 
In graphs, 
nodes and edges represent entities and their relationships respectively.
To enrich the information of graphs,
each node is usually associated with a descriptive label.
For example, 
each user in Facebook has interest as its label.
Recently,
to learn from graphs,
graph neural networks (GNNs) have achieved significant success,
which can seamlessly integrate node features with graph structure to learn node representations.
GNNs generally adopt the message passing strategy to
aggregate information from nodes' neighbors and then update node representations.
After that, 
the learned representations are fed into various downstream tasks, such as node classification.

Despite their great performance, GNNs still encounter challenges from label noise, which is a pervasive issue in real-world datasets.
For example, 
non-expert sources like Amazon's Mechanical Turk and web pages can be used to collect labels, 
which could unintentionally lead to 
unreliable labels and noise.
It is well known that deep neural networks are very sensitive to noise, 
thereby degrading their generalization ability~\cite{zhang2021understanding}. 
To address the problem,
there have been a majority of 
neural network models~\cite{patrini2017making,huang2020self,song2019selfie,yu2019does}
proposed in the field of computer vision (CV)
to tackle label noise in images.
Unfortunately, we cannot
directly apply these approaches 
to graph-structured data
as graphs are non-Euclidean data types.
Therefore, exploring the issue of graph label noise is highly necessary.

\begin{figure}
  \centering  \includegraphics[width=1.0\linewidth]{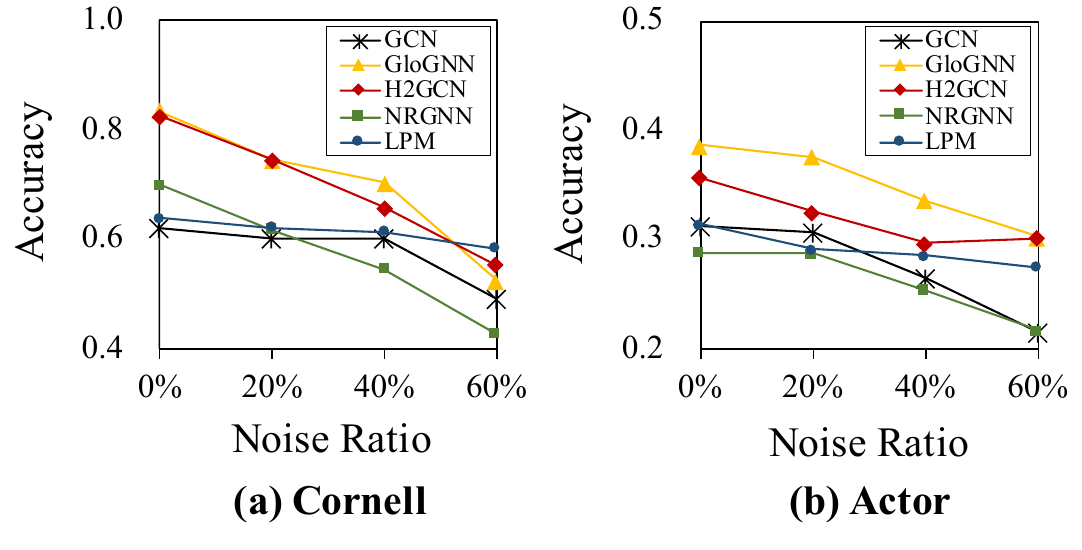}
   \vspace{-2em}
  \caption{The classification accuracy with flip label noise ranging from 0\% to 60\% on Cornell and Actor datasets.
  }
  \label{fg:hete_drop}
\vspace{-1em} 
\end{figure}

Recently, very few works have been proposed to deal with graph label noise~\cite{LPM, Nrgnn}.
These models still 
attenuate graph label noise based on an explicit \emph{homophily} assumption that 
linked nodes in a graph are more likely to be with the same label.
However, 
homophily is not a universal principle in the real world.
In fact, 
graph datasets often exhibit varying degrees of homophily, 
and may even be heterophily-dominated.
Methods based on the {homophily} assumption evidently cannot handle these heterophilous graphs.
As shown in Figure~\ref{fg:hete_drop}, we compare the performance of GCN, H$_{2}$GCN~\cite{H2GCN} and GloGNN~\cite{li2022finding}, which are specially designed for homophilous and heterophilous graphs,
against two state-of-the-art methods for graph label noise, LPM~\cite{LPM} and NRGNN~\cite{Nrgnn}, 
on the heterophilous benchmark dataset \emph{Cornell} and \emph{Actor} with {flip noise}\footnote{Flip noise and uniform noise are explained in Section~\ref{sec:es}.} ranging from 0\% to 60\%. The results indicate that on heterophilous graphs, 
current methods for label noise (LPM and NRGNN), are still underperformed by GNN models designed for graph heterophily (H$_{2}$GCN and GloGNN).
The problem of graph label noise deserves further investigation.


In this paper, we study graph label noise in the context of arbitrary heterophily, with the aim of rectifying noisy labels and assigning labels to previously unlabeled nodes.
We first explore the impact of graph homophily on graph label noise by conducting a series of experiments on benchmark heterophilous datasets. Specifically, we manipulate graph homophily by graph reconstruction to see if existing methods could improve their classification performance against label noise.
The observations demonstrate that a high graph homophily rate can indeed mitigate the effect of graph label noise on GNN performance.
In addition,
we find that label propagation (LP)~\cite{lp} based methods achieve great performance as graph homophily increases.

These inspirations led us to propose a 
effective algorithm, \ours, incorporating graph \textbf{R}econstruction for homophily and noisy label \textbf{R}ectification by \textbf{L}abel \textbf{P}ropagation. 
We further refine this basic solution in several aspects. In terms of effectiveness, we implement a multi-round selection process as opposed to correcting noisy labels all at once. Specifically, we select high-confidence noisy labels to augment the correct label set, and then repeat the algorithm, achieving higher accuracy. 
With respect to efficiency, we reduce the time complexity to linear time by unifying the computation of graph construction and label propagation. Ultimately, we provide a theoretical proof of \ours, illustrating its impressive denoising effect in addressing graph label noise.


To summarize, our main contributions are:

\noindent$\bullet$ \textbf{Empirical Analysis:} 
We first investigate the influence of graph homophily on label noise through the manipulation of homophily levels and noise ratios within real-world datasets. 
Moreover, we incorporate three common homophilous graph reconstruction modules into existing approaches, to evaluate the performance improvement of these approaches in heterophilous graphs.


\noindent$\bullet$ \textbf{Algorithm Design and Theoretical Proof:}
Based on empirical observations, 
we propose a 
effective algorithm \ours\ to deal with the issue of graph label noise.
Specifically, 
\ours\ is an iterative algorithm with three steps: (1) reconstruct the graph to recover the homophily property, (2) utilize label propagation to correct the noisy labels, (3) select high-confidence corrected labels and add them to the clean label set to retain for the next iteration. By repeating these steps, we obtain a set of ``correct'' labels to achieve accurate node classification.
We also provide a theoretical analysis of the denoising effect of \ours.

\noindent$\bullet$ \textbf{Extensive Experiments:} 
We conduct extensive experiments on two homophilous datasets and eight heterophilous datasets with varying noise types and ratios. Our proposed \ours outperforms baselines in the majority of cases, demonstrating the superior performance of \ours. The ablation studies are also done to verify the necessity of each component of \ours.

\section{Preliminaries}
\textbf{Problem Definition.}
We consider the semi-supervised node classification task in an undirected graph 
$\mathcal{G} = \mathcal{(V, E)}$
where $\mathcal{V} = \left \{  v_{i} \right \}_{i=1}^{n} $ is a set of nodes
and $\mathcal{E\subseteq V\times V }$ is a set of edges.
Let the label set as $\mathcal{Y}$ with $c$ classes, and $\mathcal{V}= \mathcal{C}\cup \mathcal{N}\cup \mathcal{U}$, where $\mathcal{C}$ is set of labeled nodes with clean labels $Y_{\mathcal{C}}$, $\mathcal{N}$ is a set of labeled nodes with noise labels $Y_{\mathcal{N}}$ and $\mathcal{U}$ is set of unlabeled nodes.
In the paper, $Y_{\mathcal{C}}$ denotes the set of known clean labels, while $Y_{\mathcal{N}}$ represents the collection containing a certain proportion of noisy labels. 
Both of them are given as inputs. It is unknown which labels in $Y_{\mathcal{N}}$ are erroneous.
Let $A$ be the adjacency matrix of $\mathcal{G}$ such that 
$A_{ij} =1$ if there exists an edge between nodes 
$v_{i}$ and $v_{j}$; 0, otherwise.
For the initial node feature matrix, we denote it as $X$.
Our goal is to design an effective algorithm 
to predict the class label for noisy and unlabeled nodes. 

\noindent\textbf{GNN and Label Propagation Basics.}
Most graph neural networks follow the message-passing mechanism, which consists of two main steps: aggregation and update. GNNs first aggregate information from neighboring nodes $\hat{h} _{i}^{(l)}=\text{AGGREGATE}(h_{j}^{(l-1)}, \forall v_{j}\in \mathcal{N}_{i})$, and then update each node embeddings $h_{i}^{(l)}=\text{UPDATE}(h_{i}^{(l-1)}, \hat{h} _{i}^{(l)})$.  
After $L$ aggregations, the final node embedding $H^{(L)}$ is obtained and then fed into a softmax layer to generate probability logits for label prediction.
Different from GNNs, label propagation directly aggregates the labels from neighboring nodes through the formula $Y_{i}^{(l)}= (1-\alpha) \cdot Y_{i}^{(l-1)}+\alpha \cdot \text{AGGREGATE}(Y_{j}^{(l-1)}, \forall v_{j}\in \mathcal{N}_{i})$. It converges when labels do not change significantly between iterations.
Depending on the various utilization of labels, the robustness of GNNs and LP to graph label noise is also different.

\begin{figure*}[t]
    \centering
    \includegraphics[scale=0.45]{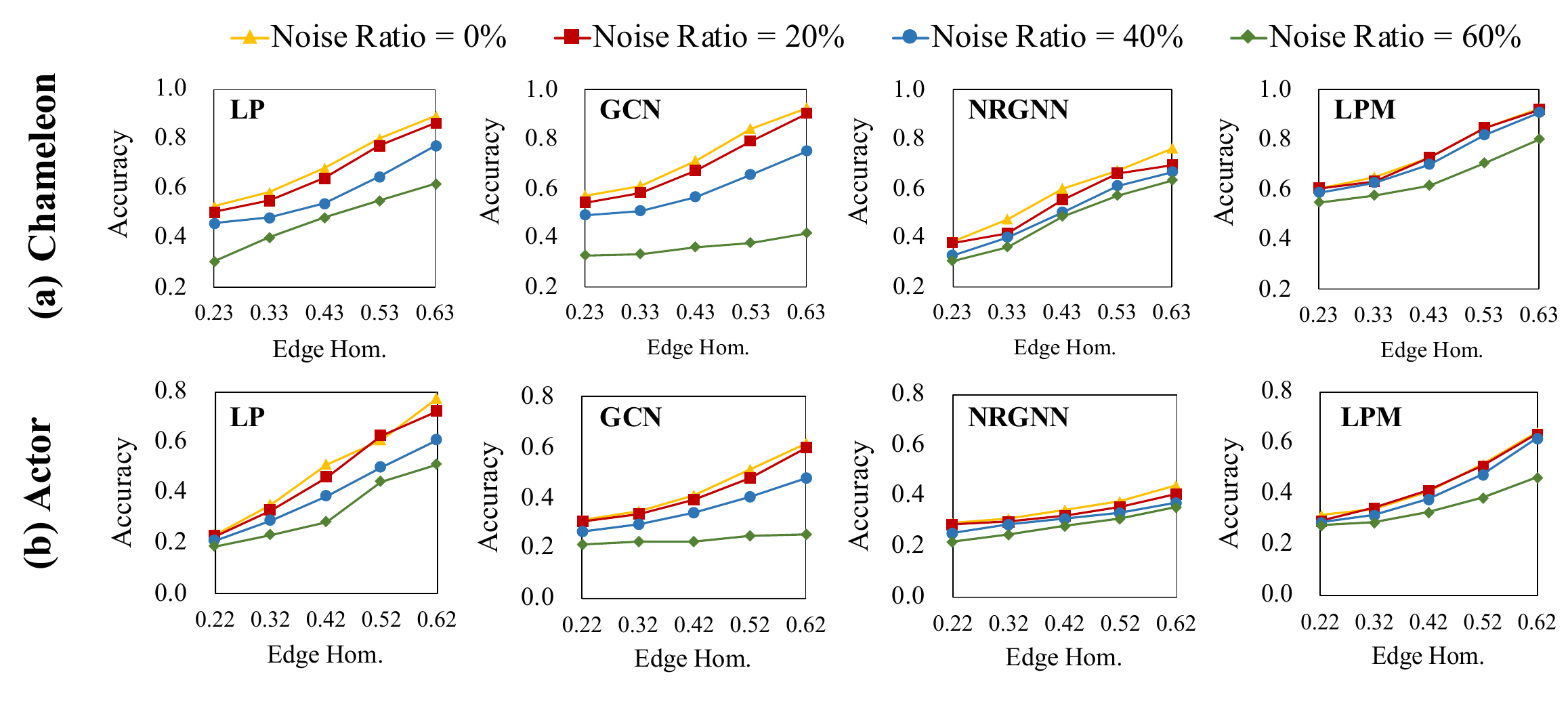}
   \vspace{-0.5em}
    \caption{The impact of edge homophily on graph label noise across various methods (flip noise)}
    \label{fg:homo_noise_flip}
 \end{figure*}

\section{How does homophily impact graph label noise?}

In this section, we investigate the potential of mitigating the effect of graph label noise by decreasing the graph heterophily rate. To do this, we first outline existing approaches for graph label noise. 
Next, we introduce three graph reconstruction modules to transform the original graphs into homophilous ones.
Ultimately, we modify the graph structure in heterophilous datasets through two manipulations: (1) converting the heterophilous edges into homophilous ones based on the ground-truth labels and (2) reconstructing the homophilous graph. We then apply the existing approaches to the modified datasets and several critical observations are derived from these experiments.


\subsection{Existing Approaches}
Here we introduce two state-of-the-art methods for graph label noise, LPM~\cite{LPM} and NRGNN~\cite{Nrgnn}.

\noindent$\bullet$ \textbf{LPM} is proposed to tackle graph label noise by combining label propagation and meta-learning. It consists of three modules: a GNN for feature extraction, LP for pseudo-label generation and an aggregation network for merging labels.
Initially, the similarity matrix $A'_{ij} =\frac{A_{ij} }{d(h_{i}, h_{j})+\tau}$ is calculated based on the adjacency matrix $A$ and the node embeddings $h_{i}$ and $h_{j}$.
These embeddings are learned by the GNN feature extraction. The function $d(\cdot, \cdot)$ represents a distance measure, and $\tau$ is an infinitesimal constant. Then LP is performed on $A'$ to generate pseudo-labels. Lastly, the aggregation network is used to integrate the original labels and pseudo-labels, resulting in the final labels through meta-learning.

\noindent$\bullet$ \textbf{NRGNN} is also designed to deal with graph label noise and label sparsity, including three modules: an edge predictor, a pseudo-label predictor, and a GNN classifier.
The edge predictor first generates extra edges to connect similar labeled and unlabeled nodes by the link prediction task. 
As a result, the newly generated adjacency matrix $A'$ becomes denser. Then the pseudo-label predictor estimates pseudo-labels $Y'$ for unlabeled nodes. The GNN classifier outputs the final predictions based on the updated $A'$ and $Y'$. 

Formally,
the overall process of these two methods can be formulated into three steps:
\begin{equation} 
\begin{aligned}
\label{eq:current_three_func}
  A' &= f_\text{reconstruct}(A, X), \\ 
  Y' &= f_\text{pseudo-label}(Y, A'), \\
  Y_\text{final} &= f_\text{predict}(X, A',Y, Y').
\end{aligned}
\end{equation}
Note that
both LPM and NRGNN depend on the graph homophily assumption. 
Further, they mainly concentrate on pseudo-label generation, while overlooking the correction of noisy labels.

\begin{figure*}[t]
   \centering  
   \includegraphics[scale=0.4]{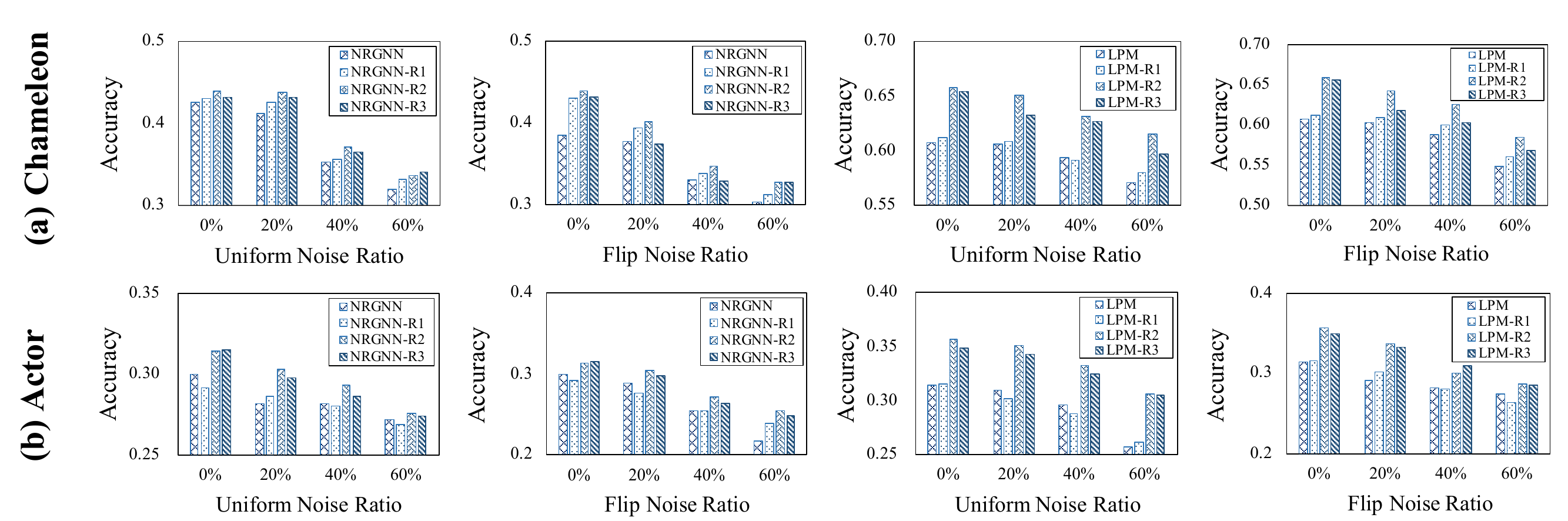}
   \vspace{-2em}
  \caption{The performance of LPM and NRGNN incorporated with three graph reconstruction modules}
  \label{fg:model_plus}
\end{figure*}

\subsection{Graph Reconstruction with Homophily}
\label{section:Graph_Reconstruction}
While both LPM and NRGNN reconstruct the original graph,
their methods have clear limitations.
For LPM,
it only adjusts the weights of raw edges in the graph, which cannot add new edges.
Further,
NRGNN employs link prediction as objective to predict new edges,
which takes existing edges 
in the graph as positive samples.
This restricts its wide applicability 
in graphs with heterophily.
Therefore,
we next introduce three common methods to convert heterophilous graphs into homophilous ones.


First, the primary goal is to directly construct a similarity matrix that serves as the homophilous graph.
Specifically, 
we train a GNN $f(\cdot)$ to aggregate node features and topological structure, 
and obtain the hidden representation $h_i$ for node $v_i$. 
We then proceed to calculate the distance between nodes by utilizing the cosine similarity metric.
The existence of edge $Z_{ij}$ between nodes $v_i$ and $v_j$
is decided by a fixed threshold $\epsilon$:
\begin{equation} 
\label{eq:z_sim}
   Z_{ij} = 1 \text{  if distance}(h_i, h_j)< \epsilon; 0, \text{otherwise}
\end{equation}
This method can be considered as a direct extension to the one used by LPM. 
On the one hand,
both methods compute the similarity between nodes based on node embddings learned by GNN.
On the other hand, 
our method can further add new edges into the graph,
while LPM cannot.

The second option is to learn a new graph adjacency matrix.
The representative method is GloGNN~\cite{li2022finding},
which learns the hidden representation $H$ for nodes and the new homophilous graph $Z$ together.
Specifically, in the $l$-th layer, it aims to obtain a coefficient matrix $Z^{(l)} \in \mathbb{R}^{n\times n}$ and node embeddings $H^{(l)} \in \mathbb{R}^{n\times d}$
satisfying  $H^{(l)} \approx  Z^{(l)}H^{(l)}$.
Each element $Z_{ij}^{(l)}$ represents how well node $v_j$ characterizes node $v_i$. The objective is formulated as:
\begin{equation}  \small
    \min\limits_{Z^{(l)}} \Vert H^{(l)} - (1-\gamma)Z^{(l)}H^{(l)} - \gamma H^{(0)}\Vert_F^2 +\beta_1\Vert Z^{(l)}\Vert_F^2+\beta_2\Vert Z^{(l)} - \sum_{k=1}^{K}\lambda_k\hat{A}^{k}\Vert_F^2,
    \label{eq:obj}
\end{equation}
where $\beta_1$ and $\beta_2$ are two hyper-parameters to control term importance. $H^{(0)}$ could be a simple transformation of initial node features $X$ and adjacency matrix $A$ (e.g., $\text{MLP}(\text{Concat}(X,A))$), and $\gamma$ is to balance the importance of initial and current embeddings.
$\hat{A}^{k}$ denotes the $k$-hop sub-graph, and $K$ is the pre-set maximum hop count.
Equation~\ref{eq:obj} has a closed-form solution:
\begin{equation} \small
\label{eq:z*}
\begin{aligned}
    Z^{(l)*} = &\left[(1-\gamma)H^{(l)}(H^{(l)})^\mathsf{T} + \beta_2 \sum_{k=1}^K \lambda_k \hat{A}^k - \gamma(1-\gamma)H^{(0)}(H^{(l)})^\mathsf{T} \right] \\
    \cdot&\left[(1-\gamma)^2H^{(l)}(H^{(l)})^\mathsf{T}+(\beta_1+\beta_2)I_n\right]^{-1}.
\end{aligned}
\end{equation}
Here, 
$Z^{(l)*}$ characterizes the connections between nodes.
After $L$ layers,
we use $Z^{(L)*}$ to denote homophilous graph structure.

In the third option, $Z^{(L)*}$ can be further adjusted to be
symmetric and non-negative.
Formally, we add the following formula after the calculation of Equation~\ref{eq:z*}, and still optimize the objective in Equation~\ref{eq:obj}:
\begin{equation}\small
\label{eq:z_hat}
\hat{Z}^{(L)} = \text{ReLU}(\frac{1}{2}(Z^{(L)*}+(Z^{(L)*})^\mathsf{T}).
\end{equation}



\subsection{Empirical Analysis}

We next investigate the relationship between graph homophily and graph label noise using two types of empirical experiments. In the first experiment, we directly enhance the homophily level by converting heterophilous edges into homophilous ones 
based on the ground-truth labels, and then run existing methods to compare their performance.
In the second experiment, we integrate the three homophilous graph reconstruction modules with existing approaches and evaluate them in heterophilous datasets.
\subsubsection{Graph Homophily Mitigates the Effect of Graph Label Noise.}
We modified the edge homophily level and label noise level on two real-world heterophilous datasets, \emph{Chameleon} and \emph{Actor}. In order to maintain the same sparsity of the dataset, we randomly replace a heterophilous edge in the original graph with a homophilous one.
We conducted experiments on four methods, 
including two classical methods: LP and GCN; and two SOTA methods for graph label noise: NRGNN and LPM.
Figure~\ref{fg:homo_noise_flip} shows the performance of all the methods under different levels of flip label noise. 
The results for uniform label noise are given in Figure~\ref{fg:homo_noise_uniform} in Appendix~\ref{sec:homo_noise_uniform}.

From the figures,
we observe that: 
(1) 
For all the methods and noise rates, 
the classification accuracy consistently improves with the increase of graph homophily on both Chameleon and Actor datasets.
Surprisingly, 
when graph homophily gets larger,
even under a high noise rate (e.g., 0.6 for LP, NRGNN and LPM, 0.4 for GCN),
these methods
can still achieve better performance 
than their corresponding results in the original heterophilous graph without label noise. 
This indicates that graph homophily can effectively mitigate the negative impact of label noise.
(2) When the noise rate is 0.6,
GCN performs poorly,
while
LP can still achieve the best results among all the methods in some cases.
This shows that 
although both GCN and LP
apply the smoothing operation on node features and labels, respectively,
GCN is adversely affected by the coupling of node features and learnable transformation matrix, which is also pointed out in~\cite{APPNP}.
This insight also demonstrates that LP excels at correcting noisy labels through label smoothing under various conditions of graph homophily.
We will give theoretical proofs in Sec.~\ref{sec:theo} later.
To sum up, all the results empirically show that graph homophily plays a significant role in denoising labels.

\subsubsection{Integration of Homophilous Graph Reconstruction and Existing Approaches.}

In the previous section, we show the importance of graph homophily. 
However, the experiments use ground-truth labels to change the level of graph homophily, 
which is not practical. 
Instead, we adapt three kinds of reconstruction modules (denoted as \textbf{R1}, \textbf{R2} and \textbf{R3}) as introduced in Sec.~\ref{section:Graph_Reconstruction} to LPM and NRGNN, and evaluate their performance in two heterophilous datasets under different noise ratios. 
In detail, we directly replace the computation of $A'$, and keep $Y'$ and $Y_\text{final}$ unchanged for LPM and NRGNN in Equation~\ref{eq:current_three_func}. 
For different reconstruction modules (\textbf{R1}, \textbf{R2} and \textbf{R3}),
$A'$ is replaced by $Z$ in Equation~\ref{eq:z_sim}, 
$Z^{(L)*}$ in Equation~\ref{eq:z*},
or $\hat{Z}^{(L)}$ in Equation~\ref{eq:z_hat}.
Figure~\ref{fg:model_plus} summarizes the results of all the variants.

From the figures,
we see that  
(1) 
The replacement of homophilous graph reconstruction modules can generally improve the performance of LPM and NRGNN in heterophilous datasets.
(2)
{The reconstruction module \textbf{R1}}, which is based on cosine similarity, only considers node representations learned by the GNN model, 
and employs a hard threshold to determine the existence of edges. 
It heavily relies on the selected GNN and the pre-set threshold. Consequently, it
cannot provide consistent improvement for both LPM and NRGNN.
(3)
The reconstruction modules \textbf{R2} and \textbf{R3}
can clearly enhance the performance of LPM and NRGNN.
This is because they directly learn the graph structure. 
Interestingly, the overall performance of \textbf{R2} surpasses that of \textbf{R3}, suggesting that
$Z$ derived in a closed-form solution
can better capture 
the true relations between nodes.
To summarize, 
incorporating graph reconstruction with homophily can effectively enhance existing approaches for graph noise labels. Among the tested modules, the second reconstruction module computed by Equation~\ref{eq:z*} demonstrates the best overall performance.
We will employ it in our following experiments.

\section{\ours: Graph Reconstruction and Noisy Label Rectification by Label Propagation}
Based on the above empirical analysis, 
we have identified two key findings.
First, 
graph homophily can mitigate the effect of graph label noise. 
Second, 
label propagation is a simple yet effective approach
to rectify incorrect labels when graph homophily is high.
These findings inspire us to utilize 
homophily graph reconstruction and label propagation, 
and propose a new approach \ours against arbitrary heterophily and label noise. 




\begin{algorithm}[h]
\small
\caption{\ours}
\label{ag}
    \LinesNumbered 
    \KwIn{
    Graph $\mathcal{G} = ({\mathcal{V},A, X)}$, $\mathcal{V}=
    \mathcal{C}\cup \mathcal{N}\cup \mathcal{U}$, 
    $Y_{\mathcal{C}}$, $Y_{\mathcal{N}}$,
    multi-round epoch $T$, selection ratio $\varepsilon$.}
    \KwOut{The label of unlabeled nodes $Y_{\mathcal{U}}$.}
    \For{$t=0,1,...,T-1$}{
        Generate $Z^{(L)*}$ on $\mathcal{C}$ by Eq.~\ref{eq:z*} \\
        Train a GNN on $Z^{(L)*}$ and $\mathcal{C}$ and calculate the predicted labels $Y_{\mathcal{P}}$ for all the nodes \\
        Utilize the label propagation to rectify noisy labels by Eq.~\ref{eq:f*} \\
        select $\varepsilon|\mathcal{N}|$ nodes from $\mathcal{N}$ based on the confidence of $F^{*}$, and then transfer them from $\mathcal{N}$ to $\mathcal{C}$\\
    }
    Train a GNN on the final clean set $\mathcal{C}$, and obtain the predicted labels $Y_{\mathcal{U}}$ \\
    \Return $Y_{\mathcal{U}}$
\end{algorithm}

\subsection{Basic version}
Assume we have a graph $\mathcal{G}$ with arbitrary heterophily, and a set of clean labels $Y_{\mathcal{C}}$ and noisy labels $Y_{\mathcal{N}}$. Our goal is to obtain the correct labels for both noisy and unlabeled nodes. As previously designed, we first reconstruct the graph structure explicitly and compute the matrix $Z^{(L)*}$ by Equation~\ref{eq:z*}. Simultaneously, the node representation $H^{(L)}$ is obtained. 
The training of the reconstruction module relies on the clean labels $Y_{\mathcal{C}}$, which allows us to establish a mapping from $H^{(L)}$ to $Y$ to generate the predicted labels $Y_{\mathcal{P}}$ for all unlabeled nodes. 
While 
$Y_{\mathcal{P}}$ is noisy,
it still contains 
rich correct labels that are useful in label propagation.


Label propagation~\cite{lp, zhou2003learning} is a graph-based method to propagate label information across connected nodes. The fundamental assumption of LP is label smoothness, which posits that two connected nodes tend to share the same label. Thus, propagating existing labeled data (including clean/noisy/predicted labels) throughout the graph can effectively help us rectify noisy labels and assign appropriate labels to previously unlabeled nodes. 

Here we set the similarity graph $S = {Z}^{(L)*}$, and $F \in \mathbb{R}^{n\times c}$ as the soft label matrix for nodes. 
When $t = 0$, 
we establish the initial label matrix 
$F(0)= \alpha_{2}Y_{\mathcal{C}} + \alpha_{3}Y_{\mathcal{N}} + \alpha_{4}Y_{\mathcal{P}}$.
Note that
in addition to the clean labels $Y_\mathcal{C}$,
both $Y_\mathcal{N}$ and $Y_\mathcal{P}$ could be noise-corrupted.
Hence, we introduce three hyper-parameters $\alpha_2$, $\alpha_3$ and $\alpha_4$ to control the reliability of these initial labels.
After that,
in the $t$-th iteration,  
LP can be formulated as: 
\begin{equation} 
    \label{eq:lps}
    F(t+1) = \alpha_{1} SF(t)+     \alpha_{2}Y_{\mathcal{C}} + \alpha_{3}Y_{\mathcal{N}} + \alpha_{4}Y_{\mathcal{P}},
\end{equation}
where $\sum_{j=1}^4\alpha_j=1$ and each $\alpha_j \in [0,1]$.
With iterations in Equation~\ref{eq:lps},
we derive: 
\begin{equation} \small
    F(t)=(\alpha_{1}S)^{t-1}F(0)
    + \alpha_{2}\sum_{i=0}^{t-1}(\alpha_{1}S)^{i}Y_{\mathcal{C}}
    + \alpha_{3}\sum_{i=0}^{t-1}(\alpha_{1}S)^{i}Y_{\mathcal{N}}
    + \alpha_{4}\sum_{i=0}^{t-1}(\alpha_{1}S)^{i}Y_{\mathcal{P}}.
\end{equation}
Hence, we can also have a closed-form solution:
\begin{equation}
\label{eq:close}
    F^{*}=\lim_{t \to \infty} F(t) = (I - \alpha _{1}S )^{-1} (\alpha_{2}Y_{\mathcal{C}}
    +\alpha_{3}Y_{\mathcal{N}}
    +\alpha_{4}Y_{\mathcal{P}}).
\end{equation}

\subsection{Efficiency Improvement}

Directly calculating $ F^{*}$ by Equation~\ref{eq:close} is computationally infeasible due to the cubic time complexity involved in computing
$ Z^{(L)*} $ in Equation~\ref{eq:z*} and the inverse term $(I - \alpha _{1} S )^{-1}$.
However, even quadratic time complexity is still high for large graph datasets. Therefore, we aim to accelerate the whole computation to achieve linear time complexity.

First of all, we utilize the Woodbury formula~\cite{1950Inverting} to replace the inverse and transform Equation~\ref{eq:z*} into:
\begin{equation} \scriptsize
\label{eq:wood}
\begin{aligned}
Z^{(l)*} = &\left [ (1-\gamma )H^{(l)} (H^{(l)})^{\mathsf{T}} + \beta_{2} \sum_{k=1}^{K}\lambda _{k}\hat{A}^{k}
    -\gamma (1-\gamma ) H^{(0)} (H^{(l)})^{\mathsf{T}}  \right ] \cdot  \\
 &\left [ \frac{1}{\beta _{1}+ \beta _{2}} I_{n}  - \frac{1}{(\beta _{1}+ \beta _{2})^{2} }
    H^{(l)}\left [ \frac{1}{(1-\gamma)^{2}}I_{d} +\frac{1}{\beta _{1}+ \beta _{2}} 
    (H^{(l)})^{\mathsf{T}}H^{(l)}\right ]^{-1}  
    (H^{(l)})^{\mathsf{T}} \right ]
\end{aligned}
\end{equation}
After that, we use first-order Taylor expansion and
derive
\begin{equation}
\label{eq:f}
    F^{*} = (I - \alpha _{1} S )^{-1}F(0)\approx (I + \alpha _{1} S)F(0).
\end{equation}
The inverse  term in Equation~\ref{eq:wood} is computed on a matrix in $\mathbb{R}^{d\times d}$,
whose 
time complexity is only $O(d^3)$.
In this way,
the overall time complexity of computing $Z^{(L)*}$ is $O(n^2)$, leading to a quadratic time complexity of calculating $F^*$ in Equation~\ref{eq:f}.

Next, we combine Equations~\ref{eq:wood} and~\ref{eq:f}:
\begin{equation}\scriptsize
\label{eq:f*} 
\begin{split}
    F^{*} = F(0) +\alpha _{1} \left [  (1-\gamma )H^{(L)}(H^{(L)})^{\mathsf{T}} +\beta _{2}\sum_{k=1}^{K}\lambda _{k}
\hat{A}^{k} -\gamma(1-\gamma )H^{(0)}(H^{(L)})^{\mathsf{T}}\right ] Q, \\ 
\end{split}
\end{equation}
where 
\begin{equation}\small
\begin{aligned}
Q = \frac{1}{\beta _{1}+ \beta _{2}} F(0) & - \frac{1}{(\beta _{1}+ \beta _{2})^{2} }  H^{(L)} \\
& \cdot [ \frac{1}{(1-\gamma)^{2}}I_{d} + \frac{1}{\beta _{1}+ \beta _{2}} 
(H^{(L)})^{\mathsf{T} }H^{(L)} ]^{-1}  
(H^{(L)})^{\mathsf{T} } F(0). 
\end{aligned}
\end{equation}
In this way, we avoid computing $Z^{(l)*}$ and $(I + \alpha _{1} S )$
explicitly, and it can be further accelerated by matrix multiplication reordering.
We first calculate $Q$
in a right-to-left manner.
Due to the 
time complexity $O(d^3)$ of computing the inverse term,
the overall time complexity of calculating $Q$ is $O(ncd + d^3)$, where $cd \ll n$ in the large datasets.
After that,
we calculate $F^*$ in Equation~\ref{eq:f*} in a similar right-to-left manner.
For example,
for the term $H^{(0)}(H^{(L)})^\mathsf{T} Q$,
we first compute $(H^{(L)})^\mathsf{T} Q$.
Then we left-multiply the result with $H^{(0)}$,
leading to a time complexity of $O(ncd)$.
When computing the term $\sum_{k-1}^K\lambda_k\hat{A}^kQ$,
we first calculate $\hat{A}Q$.
Due to the sparsity of $\hat{A}$,
the time complexity is $O(pcn)$,
where $p$ is the average number of non-zero entries in each row of $\hat{A}$.
The total time complexity for $\sum_{k-1}^K\lambda_k\hat{A}^kQ$ is $O(Kpcn)$, where $Kpc$ is a coefficient.
Finally,
we approximate the optimal label propagation results in a linear time.

\subsection{Effectiveness Enhancement}
\label{sec:sample_select}

To further enhance the effectiveness of \ours, we employ a multi-round selection strategy. 
Specifically, $F^*$ is obtained after the computation of graph reconstruction and label propagation. We then calculate the confidence based on $F^*$ for noisy nodes, select the rectified nodes with the highest confidence and add them to the clean label set $Y_\mathcal{C}$. 
In each time, $\varepsilon * \vert\mathcal{N}\vert$ nodes are chosen, where $\varepsilon$ is the pre-set ratio. The updated clean label set can be leveraged in the next round to reconstruct the graph and refine the remaining nodes.
Finally, we obtain an augmented clean label set, upon which we train a GNN classifier to obtain the node classification result.
The pseudo-code of \ours\ is provided 
in Algorithm~\ref{ag}.

The input is a graph with the topology $A$ and node features $X$, and we divide the nodes into three sets $\mathcal{C}$, $\mathcal{N}$ and $\mathcal{U}$. Here, $\mathcal{C}$ represents clean labeled nodes, $\mathcal{N}$ denotes noisy labeled nodes, and $\mathcal{U}$ represents unlabeled nodes. We have the labels $Y_{\mathcal{C}}$ and $Y_{\mathcal{N}}$, and we aim to predict $Y_{\mathcal{U}}$. The whole process is a multi-round training to select the noisy nodes with corrected labels into the clean node set $\mathcal{C}$. In each round, we first reconstruct the graph with homophily based on the current clean labels, and train a GNN to obtain the predicted labels for all the nodes as extra information. Next, we employ label propagation to rectify noisy labels.
Given the selection ratio $\varepsilon$, we select a portion of the current noisy nodes based on the confidence of $F^{*}$. Lastly, we transfer the selected noisy nodes with their corrected labels from $\mathcal{N}$ to $\mathcal{C}$ and initiate a new round. \ours\ typically runs for 10-20 iterations. After we end all the rounds, we obtain the final clean set $\mathcal{C}$ to train a GNN and predict the labels for unlabeled nodes.


\subsection{Theoretical Verification}
\label{sec:theo}
In this section, we theoretically analyze the denoising effect of label propagation.
For clarity, we first assume the clean label is corrupted by a noise transition matrix $\bm{T} \in [0,1]^{c \times c}$, where $\bm{T}_{ij}=\mathbb{P}(\tilde{Y} = C_j|Y = C_i)$ is the probability of the label $C_i$ being corrupted into the label $C_j$. For simplicity and following~\cite{wei2021smooth}, we assume the classification is binary and the noise transition matrix is 
$\bm{T} = \begin{bmatrix}
1-e & e \\
e & 1-e
\end{bmatrix}$. After label propagation, the propagated label of the node $v_i$ becomes a random variable 
\begin{align}\label{eq:lp} \small
\hat{Y}_{i} = (1-\alpha) \tilde{Y}_{i} + \frac{\alpha }{d} \sum_{v_j \in \mathcal{N}_i} \tilde{Y}_{j}. 
\end{align}
which is decided by its own noisy label $\tilde{Y}_{i}$ and its $d$ neighbors' noisy labels $\tilde{Y}_{j}$. We establish the following theorem for the propagated label $\hat{Y}_{i}$, 
and the detailed proof can be found in Appendix \ref{app:prof_denoise}.

\begin{theorem}\label{thm:denoise} (Label Propagation and Denoising) Suppose the label noise is generated by $\bm{T}$ and the label propagation follows Equation~\ref{eq:lp}. For a specific node $i$, we further assume the node has $d$ neighbors, and its neighbor nodes have the probability $p$ to have the same true label with node $i$, i.e., $\mathbb{P}[Y_i = Y_j] = p$. After one-round label propagation, the gap between the propagated label and the ground-true label is:
\begin{equation} \small
\begin{aligned}\label{eq:denoising_1} 
    \mathbb{E}(Y-\hat{Y})^2 &= \mathbb{P}(Y=0)\mathbb{E}(\hat{Y}|Y=0)^2  \\
    & + \mathbb{P}(Y=1)\left(\mathbb{E}(\hat{Y}| Y =1) - 1 \right)^2+ \text{Var}(\hat{Y}),
\end{aligned}
\end{equation}
where we have
\begin{subequations} 
\begin{align}
     &\mathbb{E}(\hat{Y}|Y=0) = (1-\alpha)e+\alpha[pe+(1-p)(1-e)],\label{eq:e0} \\
    &\mathbb{E}(\hat{Y}|Y=1) = (1-\alpha)(1-e)+\alpha[p(1-e)+(1-p)e], \label{eq:e1} \\
    &\text{Var}(\hat{Y}) =  (1-\alpha)^2e(1-e) \\
    & \quad \quad \quad \quad +\frac{\alpha^2 }{d} [pe+(1-p)(1-e)] [1-pe-(1-p)(1-e)]. \label{eq:var}
\end{align}
\end{subequations}
\end{theorem}
(1) \textbf{Heterophily and the impact of $p$:} In Equation~\ref{eq:denoising_1}, $p$ plays the central role as a larger $p$ can reduce all three terms in the bound. The value of $p$ is related to the heterophily of \emph{a specific node} and a more heterophilous node tends to have a smaller $p$. This is our major motivation to construct $Z^{(L)*}$ in the algorithm, which attempts to connect the nodes from the same class. 

(2) \textbf{Label denoising effect:} We consider two special cases to see the label denoising by propagation. When we set $\alpha = 0$ and do not propagate the labels, $\mathbb{E}(Y-\hat{Y})^2 = \mathbb{E}(Y-\tilde{Y})^2 =e$. When we set $p \to 1$ where all the neighbors have the same true labels, we have $\mathbb{E}(Y-\hat{Y})^2 \approx e^2 + \frac{1}{d}e(1-e)$, which is strictly less than $e$. It demonstrates the importance of neighbors. Besides, when we have more neighbors, the distance becomes smaller. It is the reason that $Z^{(L)*}$ is dense so each node could have more neighbors. 

(3) \textbf{The impact of clean and predicted labels:} The bound in Equation~\ref{eq:denoising_1} monotonically decreases with $e$ and therefore having clean nodes can substantially reduce the bound. When the predicted labels are present, we have a larger number of neighbors for propagation, which also increases the value $d$.

\begin{table*}[h]
\setlength{\tabcolsep}{2.5pt}
  \centering
  \caption{The classification accuracy (\%) over the methods on 10 datasets with flip noise ranging from 0\% to 60\%. We highlight the best score on each dataset in bold. OOM denotes the out-of-memory error.}
   \vspace{-1em}
  \label{flip}
  \resizebox{\linewidth}{!}{
  \begin{tabular}{ccccccccccccc}
    \toprule
    \multicolumn{1}{c|} {\textbf{Types}} & \textbf{Methods} &  \multicolumn{1}{c|}{\textbf{Noise}} &\textbf{Cora} & \textbf{Citeseer} &
    \textbf{Chameleon} & \textbf{Cornell} & \textbf{Wisconsin} &
    \textbf{Texas} & \textbf{Actor} &\textbf{Penn94} & \textbf{arXiv-year} & \textbf{snap-patents} \\ 
    \hline
    \multicolumn{1}{c|}{\multirow{12}{*}{
    GNNs
    }} & \multirow{4}{*}{GCN} &\multicolumn{1}{c|}{0\%}& 87.75 &76.81  & 56.95 & 61.89 &63.92  & 62.16 & 31.27 & 80.76 & 44.67  & 53.59\\
    \multicolumn{1}{c|}{} &  & \multicolumn{1}{c|}{20\%} & 86.08 & 75.48  & 56.27 & 60.00 & 62.94 & 61.89 & 30.63 & 78.99 & 40.93 & 49.49\\
    \multicolumn{1}{c|}{}  &  & \multicolumn{1}{c|}{40\%} & 81.30 & 70.99  & 53.81 & 60.27 & 61.17 & 61.35 & 26.38 & 75.84 & 40.20 & 41.84\\
    \multicolumn{1}{c|}{} &  & \multicolumn{1}{c|}{60\%} & 53.85 & 42.94 & 52.98 & 49.18  & 36.47 & 48.64 & 21.51 & 72.70 & 38.73 & 30.55\\ 
    \cline{2-13} 

    \multicolumn{1}{c|}{\multirow{12}{*}{}} & \multirow{4}{*}{GloGNN} &\multicolumn{1}{c|}{0\%}& 84.57 & 76.09  & 68.83 & 83.51 &  88.43& 84.05  & 38.55 &85.66 &52.74 &62.33 \\
    \multicolumn{1}{c|}{} &  & \multicolumn{1}{c|}{20\%} & 80.84 & 74.09  & 67.76 & 76.48 & 81.76 & 71.62  & \textbf{37.55} &76.97 &47.67 &54.72\\
    \multicolumn{1}{c|}{}  &  & \multicolumn{1}{c|}{40\%} &77.55  & 63.54 & \textbf{63.00 }& 70.27& 74.31& 61.08 & 33.64 &62.55 &36.93 &38.74\\
    \multicolumn{1}{c|}{} &  & \multicolumn{1}{c|}{60\%} &  75.10  & 40.17& 55.08 & 52.43 & 56.86 & 53.78 & 30.31 &51.79 &25.45 &26.37\\
    \cline{2-13}

    \multicolumn{1}{c|}{\multirow{12}{*}{}} & \multirow{4}{*}{H$_{2}$GCN} &\multicolumn{1}{c|}{0\%}& 87.87  & 77.11   &60.01    &  82.59 &  85.65 & 86.26 & 35.70  &79.02 &48.72 &57.21 \\
    \multicolumn{1}{c|}{} &  & \multicolumn{1}{c|}{20\%} & 83.51  & 74.80  &54.10 & 74.45 &71.96  &70.78  & 32.49 &74.08 &43.20 &52.06\\
    \multicolumn{1}{c|}{}  &  & \multicolumn{1}{c|}{40\%} & 82.00  & 71.42  &47.82 & 65.67& 70.60 & 67.83  &29.73 &67.42 &36.57 &46.23\\
    \multicolumn{1}{c|}{} &  & \multicolumn{1}{c|}{60\%} & 81.36   &  71.23 & 34.51 & 55.27 & 64.90 & 60.81 & 30.21 &54.47 &29.81 &38.57\\
    \hline
    
    \multicolumn{1}{c|}{\multirow{8}{*}{\makecell{Methods\\
   for \\Label \\Noise}}} & \multirow{4}{*}{Co-teaching+} &\multicolumn{1}{c|}{0\%} &  85.76& 76.10    & 72.52  & 70.54 &  72.49 &  67.91& 34.28 &86.90 &OOM &OOM\\
    \multicolumn{1}{c|}{} &  & \multicolumn{1}{c|}{20\%} &  68.79& 70.19  & 66.09 & 66.09& 61.92 & 64.70  & 30.16 &77.68 &OOM &OOM\\
    \multicolumn{1}{c|}{}  &  & \multicolumn{1}{c|}{40\%} & 55.34  & 47.87   & 41.03& 41.03& 57.88& 57.08  & 24.98&61.82 &OOM &OOM\\
    \multicolumn{1}{c|}{} &  & \multicolumn{1}{c|}{60\%} & 35.72   & 36.11  &28.57  & 28.57 & 39.60 & 38.56 &21.16  &50.18 &OOM &OOM\\
    \cline{2-13}

    \multicolumn{1}{c|}{\multirow{8}{*}{}} & \multirow{4}{*}{Backward} &\multicolumn{1}{c|}{0\%} &  84.88 & 77.06   & 71.00  & 72.70  & 81.37  & 84.31  & 23.90  &\textbf{87.00} &46.51 &58.54\\
    \multicolumn{1}{c|}{} &  & \multicolumn{1}{c|}{20\%} & 81.04 & 72.43  &67.69  &69.72  & 75.88 & 74.72 & 21.15 &80.84 &40.29 &56.30\\
    \multicolumn{1}{c|}{}  &  & \multicolumn{1}{c|}{40\%} & 71.28  & 62.03 & 61.73 & 68.91& 63.13 & 60.78  & 20.01 &70.00 &32.24 &43.61 \\
    \multicolumn{1}{c|}{} &  & \multicolumn{1}{c|}{60\%} & 56.70   & 48.37   & 53.15 & 51.35 & 45.49 &53.51  & 20.09 &51.57 &26.21 &31.57\\
    \hline

    \multicolumn{1}{c|}{\multirow{12}{*}{\makecell{Methods \\for\\ Graph\\ Label\\ Noise}}} & \multirow{4}{*}{NRGNN} &\multicolumn{1}{c|}{0\%} & 82.87 & 72.52   & 38.50 & 69.91 &  69.17 &  72.64 &  28.82  &68.31 &OOM &OOM\\
    \multicolumn{1}{c|}{} &  & \multicolumn{1}{c|}{20\%} & 82.30 & 72.47  & 37.83 & 61.70& 67.37 & 70.32& 28.75 &66.80 &OOM &OOM\\
    \multicolumn{1}{c|}{}  &  & \multicolumn{1}{c|}{40\%} &79.13  &67.42  & 33.09& 54.59 & 61.64 & 57.02  & 25.40 &63.59 &OOM &OOM  \\
    \multicolumn{1}{c|}{} &  & \multicolumn{1}{c|}{60\%} & 75.40  & 60.00 & 30.35 & 42.70 & 58.23 &54.59  &21.69 &53.14 &OOM &OOM\\
    \cline{2-13}

    \multicolumn{1}{c|}{\multirow{12}{*}{}} & \multirow{4}{*}{LPM} &\multicolumn{1}{c|}{0\%} & \textbf{89.74} & \textbf{78.77}  &  60.72 & 63.87  & 73.72  & 69.46  & 31.43  &76.10 &44.32 &56.76\\
    \multicolumn{1}{c|}{} &  & \multicolumn{1}{c|}{20\%} & 86.55  & 75.92 & 60.31&62.21 &72.35  & 68.10  & 29.11 &75.35 &42.03 &53.84\\
    \multicolumn{1}{c|}{}  &  & \multicolumn{1}{c|}{40\%} & 83.97 & 72.62 &58.77& 61.35 & 70.78 & 67.02  & 28.52 &71.45 &38.46 &48.36\\
    \multicolumn{1}{c|}{} &  & \multicolumn{1}{c|}{60\%} & 80.47  &68.72 & 54.75& 58.37& 64.11 & 62.16 & 27.51 &63.14 &32.11 &40.21\\
    \cline{2-13}

    \multicolumn{1}{c|}{\multirow{12}{*}{}} & \multirow{4}{*}{RTGNN} &\multicolumn{1}{c|}{0\%} & 86.08&	76.27&	45.88&	62.83&	67.64&	71.43&	29.55&	72.65&	OOM&	OOM\\
    \multicolumn{1}{c|}{} &  & \multicolumn{1}{c|}{20\%} & 85.24&	75.45&	39.03&	57.43&	60.78&	67.56&	28.25&	70.67&	OOM&	OOM\\
    \multicolumn{1}{c|}{}  &  & \multicolumn{1}{c|}{40\%} & 80.21&	69.85&	31.08&	47.97&	56.86&	60.13&	25.67&	60.48&	OOM&	OOM\\
    \multicolumn{1}{c|}{} &  & \multicolumn{1}{c|}{60\%} & 73.25&	61.29&	25.58&	31.75&	47.54&	50.21&	18.96&	53.98&	OOM&	OOM\\
    \cline{2-13}

    \multicolumn{1}{c|}{\multirow{12}{*}{}} & \multirow{4}{*}{ERASE} &\multicolumn{1}{c|}{0\%} & 87.32&	77.14&	50.15&	62.62&	65.52&	72.91&	30.16&	75.36&	43.27&	56.50\\
    \multicolumn{1}{c|}{} &  & \multicolumn{1}{c|}{20\%} & 86.51&	75.86&	42.53&	58.67&	58.46&	70.04&	29.85&	72.61&	40.51&	52.65\\
    \multicolumn{1}{c|}{}  &  & \multicolumn{1}{c|}{40\%} & 82.19&	70.15&	35.16&	50.11&	56.03&	65.55&	27.14&	64.67&	35.62&	46.33\\
    \multicolumn{1}{c|}{} &  & \multicolumn{1}{c|}{60\%} & 75.63&	61.22&	30.88&	33.84&	45.32&	57.47&	20.58&	58.26&	30.17&	38.04\\
    \cline{2-13}
    
    \multicolumn{1}{c|}{\multirow{12}{*}{}} & \multirow{4}{*}{\ours} &\multicolumn{1}{c|}{0\%} & 87.69 & 78.09  &\textbf{73.05}  &  \textbf{87.84}& \textbf{88.82} & \textbf{87.30} & \textbf{38.59} &86.78 &\textbf{53.80} & \textbf{63.04}\\
    \multicolumn{1}{c|}{} &  & \multicolumn{1}{c|}{20\%} &\textbf{86.85} & \textbf{76.30} & \textbf{68.38}& \textbf{84.32}&\textbf{86.86} &  \textbf{83.24}& 37.37 &\textbf{82.69} &\textbf{50.11} &\textbf{58.59}\\
    \multicolumn{1}{c|}{}  &  & \multicolumn{1}{c|}{40\%} & \textbf{85.00} & \textbf{74.41} & 61.23& \textbf{77.84}&\textbf{83.14}&  \textbf{77.02}& \textbf{35.86} &\textbf{79.05} &\textbf{46.19} & \textbf{55.90}\\
    \multicolumn{1}{c|}{} &  & \multicolumn{1}{c|}{60\%} & \textbf{82.53}  & \textbf{72.73} & \textbf{57.46}& \textbf{71.35}& \textbf{75.68}& \textbf{69.46}& \textbf{35.08}&\textbf{77.09} &\textbf{46.11} &\textbf{52.71}\\
    \bottomrule
\end{tabular}
}
\end{table*}

In the following theorem, we analyze the generalization error of the final GNN classifier $\hat{f}$. The proof is provided in Appendix \ref{app:prof_gen}.
\begin{theorem}\label{thm:gen} (Generalization Error and Oracle Inequality)
    Denote the node feature $X$ sampled from the distribution $D$, the graph topology as $A$, the set of training nodes as $S_n$, the graph neural network as $f(\cdot, \cdot)$, and the learned GNN classifier $\hat{f} = \inf_f \sum_{X\in S_n} (\hat{Y} - f(X, A))$. Suppose that the propagated label concentrated on its mean, i.e., with probability at least $\frac{\delta}{2}$, $||\hat{Y} - Y| - \mathbb{E}|\hat{Y} - Y|| \leq \epsilon_1$. We further assume the generalization error is bounded with respect to the propagated labels, i.e., with probability at least $\frac{\delta}{2}$,
    \begin{align*} \small
        \mathbb{E}_{X \sim D} |\hat{Y} - \hat{f}(X,A)| - \inf_f \mathbb{E}_{X \sim D} |\hat{Y} - f(X,A)| \leq \epsilon_2.
    \end{align*}
    Then we obtain the generalization error bound for training with noisy labels and test on the clean labels, i.e., with probability at least $\delta$, the generalization error trained with propagated labels is given by
    \begin{align*} \small
        \mathbb{E}_{X \sim D} |Y - \hat{f}(X,A)| \leq  \inf_f \mathbb{E}_{X \sim D} |Y - f(X,A)| + \epsilon_2 + 2(\mathbb{E}|\hat{Y} - Y| + \epsilon_1). 
    \end{align*}
\end{theorem}
The generalization error mainly depends on $\epsilon_2$ and $\mathbb{E}|\hat{Y}-Y|$. Our algorithm iteratively rectifies the noisy labels and selects high-confidence labels into the clean set, thereby reducing $\mathbb{E}|\hat{Y}-Y|$. In addition, as shown in \cite{dong2019distillation}, using the predicted labels generated by the same neural network is able to decrease $\epsilon_2$. This is the reason for the inclusion of predicted labels within our algorithm.

\section{Related Work}
\subsection{Deep Learning with Noisy Labels}
Dealing with noisy labels can be approached through two mainstream directions: loss adjustment and sample selection~\cite{song2022learning}.
Particularly, loss adjustment aims to mitigate the negative impact of noise by adjusting the loss value or the noisy labels. 
For example, \cite{patrini2017making} explicitly estimates the noise transition matrix to correct the forward and backward loss. Other effective techniques, such as~\cite{huang2020self,song2019selfie}, refurbish the noisy labels by a convex combination of noisy
and predicted labels. 
Sample selection, on the other hand, involves selecting true-labeled examples from a noisy training set via multi-network or multi-round learning. 
For instance, Co-teaching~\cite{han2018co} and Co-teaching+~\cite{yu2019does} use multi-network training, where two deep neural networks are maintained and each network selects a set of small-loss examples and feeds them to its peer network. Multi-round learning methods~\cite{wu2020topological, song2019selfie} iteratively refine the selected set of clean samples by repeating the training rounds. As a result, the performance improves as the number of clean samples increases.

\subsection{Learning GNNs with Noisy Labels}
Graph neural networks (GNNs)~\cite{GCN, APPNP, GAT, GIN, GraphSAGE, SGC} have emerged as revolutionary technologies for graph-structured data, which can effectively capture complex patterns based on both node features and structure information, and then infer the accurate node labels. 
It is also common that real-world graphs are heterophilous which node is more likely to be connected with neighbors that have dissimilar features or different labels. Many studies are proposed to design GNNs for graphs with heterophily~\cite{H2GCN, li2022finding, fagcn2021, chien2021adaptive}.

The robustness of GNNs is well-studied, but most of them focus on the perturbation of graph structure and node features~\cite{dai2018adversarial, wu2019adversarial, zugner2018adversarial, sun2020adversarial} while few works study the label noise. NRGNN~\cite{Nrgnn} and LPM~\cite{LPM} are the first two to deal with the label noise in GNNs explicitly. In detail, NRGNN generates pseudo labels and assigns more edges between unlabeled nodes and (pseudo) labeled nodes against label noise. LPM utilizes label propagation to aggregate original labels and pseudo labels to correct the noisy labels.
However, they are only suitable for homophilous graphs, and perform badly on graphs with heterophily. 
Our proposed approach, \ours\ relaxes the homophily assumption and is more resilient to various types of noise.

\section{Experiments}
\subsection{Experimental Settings}
\label{sec:es}

\noindent\textbf{[Label noise].}
For each dataset, we randomly
split nodes into 60\%, 20\%, and 20\%
for training, validation and testing, and measure the performance of all models on the test set.
We set 10\% nodes with clean labels, 
and corrupt the remaining training nodes with noise in training set. 
Following~\cite{patrini2017making}, 
We consider two types of noise:
\emph{uniform noise} means the true label $c_i$ have a probability $e$ to be uniformly flipped to other classes, i.e.\ $ \bm{T}_{ij} =\ e/(c - 1)$ for $i \ne j $ and $\bm{T}_{ii}=1-e$;
\emph{flip noise} means the true label $C_i$ have a probability $ e $ to be flipped to only a similar class $C_j$, i.e.\ $ \bm{T}_{ij} = e$ and $ \bm{T}_{ii} = 1 - e$.
We vary the value of $e$ from 0\% to 60\%.

\noindent\textbf{[Baselines].} 
We compare \ours with 10 baselines, which can be categorized into three types. The first type involves (1) \textbf{GCN}~\cite{GCN} (2) \textbf{GloGNN}~\cite{li2022finding} and (3) \textbf{H$_{2}$GCN}~\cite{H2GCN}, three popular GNNs for homophilous and heterophilous graphs. They are directly trained on both clean and noisy training samples. 
The second type includes baselines that handle label noise:
(4) \textbf{Co-teaching+}~\cite{yu2019does} is a multi-network training method to select clean samples. 
(5) \textbf{Backward}~\cite{patrini2017making} is to revise predictions and obtain unbiased loss on noisy training samples.
For fairness, we use GloGNN as a backbone GNN model for Co-teaching+ and Backward.
The third type includes five GNNs dealing with label noise: 
(6) \textbf{NRGNN}~\cite{Nrgnn}, 
(7) \textbf{LPM}~\cite{LPM}, 
(8) \textbf{RTGNN}~\cite{RTGNN},
(9) \textbf{ERASE}~\cite{erase},
and (10) \textbf{CGNN}~\cite{yuan2023learning}.
Since the absence of publicly available code for CGNN, we validate \ours\ under the same settings in their paper and report the results of CGNN directly from the paper.
The results are presented in the appendix~\ref{sec:cgnn} due to the space limitation.

Due to the space limitation, we move details on baselines (Sec.~\ref{ap:baselines}), 
datasets (Sec.~\ref{ap:dataset_description}), 
experimental setup (Sec.~\ref{sec:Experimental-setup}),
hyper-parameter sensitivity analysis (Sec.~\ref{sec:Sensitivity_Analysis})
, and the analysis of selection strategies (Sec.~\ref{sec:Selection_Strategies}) to Appendix.

\subsection{Performance Results}\label{sec:performance}

\begin{figure*}[ht!]
  \centering  \includegraphics[width=0.8\linewidth]{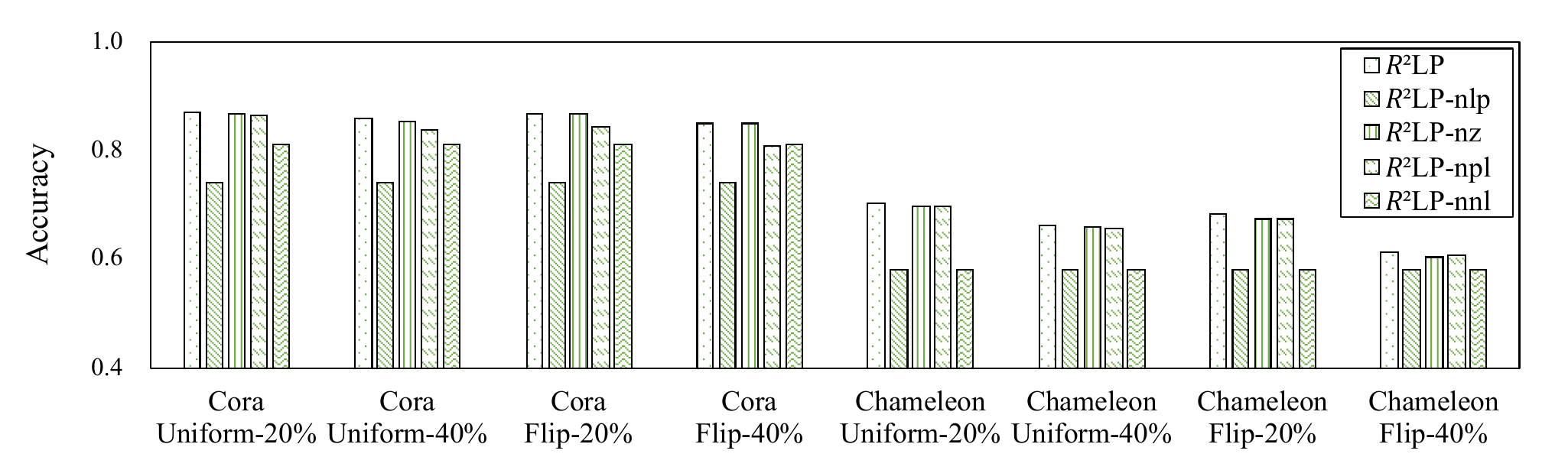}
   \vspace{-1,2em}
  \caption{Ablation study on the main components of \ours. 
  }
  \label{ablation}
\end{figure*}

\begin{table}[ht!] 
  \centering
  \caption{The classification accuracy (\%)
  without a clean label set. We highlight the best score on each dataset in bold and underline the runner-up's.}
  \vspace{-1em}
  \label{table:no_clean}
  \resizebox{\linewidth}{!}
{
  \begin{tabular}{c | c | c c c |c c c  }
    \toprule
    \multirow{2}{*}{\textbf{Datasets}} & 
    \multirow{2}{*}{\textbf{Methods}} & 
    \multicolumn{3}{c|}{\textbf{Uniform Noise}} & \multicolumn{3}{c}{\textbf{Flip Noise}} \\ 
    &  & 20\% & 40\% & 60\%   
    & 20\% & 40\% & 60\%  \\
    \midrule
    \multirow{6}{*}{Cora}
    & GloGNN & 82.53 & 79.92 & 75.70 & 80.84 & 71.06  & 41.00 \\
    & H$_{2}$GCN  & \underline{85.34} & \underline{80.62} & 75.52 & \underline{82.25} & \underline{71.20} &  45.42 \\
    & NRGNN  & 79.43 & 72.10 & 63.92 & 76.42 & 70.19 & \textbf{60.32} \\
    & Co-teaching+  & 80.00 & 69.59 & 55.78 & 75.20 & 59.25 &  33.15 \\
    & Backward & 83.35 & 80.18 & \underline{76.16} & 81.04 & 69.65 & 42.03 \\
    & \ours & \textbf{85.45} & \textbf{84.32} & \textbf{78.87} & \textbf{84.95} & \textbf{76.21} &  \underline{46.62} \\
    \midrule
    \multirow{6}{*}{Chameleon}
    & GloGNN  & 63.94 & 60.83 & \underline{58.46} & 62.19 & 52.71  & 33.35 \\
    & H$_{2}$GCN  & 55.30 & 51.79 & 46.60 & 54.10 & 47.83 & 34.36 \\
    & NRGNN  & 37.01 & 34.56 & 31.85 & 35.37 & 31.86 & 29.32 \\
    & Co-teaching+  & 66.60 & 61.90 & 50.61 & 64.40 & 40.19 & 25.94 \\
    & Backward  & \underline{69.34} & \textbf{65.46} & \textbf{61.31} & \textbf{67.54} & \underline{58.33} & \textbf{42.94} \\
    & \ours  & \textbf{69.45} & \underline{64.16} & 58.13 & \underline{67.23} & \textbf{59.01} & \underline{42.87} \\
    \bottomrule
\end{tabular}
}
\end{table}

Table~\ref{flip} presents the performance results of all the methods on 10 benchmark datasets under different flip noise levels. We attach Table~\ref{uniform} for various uniform noise levels in Appendix.
Every method was repeated 10 times for small-scale datasets and 5 times for large-scale datasets over different random splits, and we report the mean classification accuracy on the test set.
From the tables, we make the following observations:
(1) \ours consistently outperforms the competitive methods across various noise settings in the majority of cases.
(2) The GNN models GCN, GloGNN and H$_{2}$GCN exhibit good performance on homophilous graphs. However, their accuracy drops dramatically as the noise ratio increases on heterophilous graphs.
(3) LPM, NRGNN, RTGNN and ERASE aim to address the label noise in graphs, but their accuracy is limited on heterophilous graphs because they rely on the assumption of graph homophily and label smoothness.
(4) Co-teaching+ and Backward are two mainstream methods to address the label noise through sample selection and loss correction.
However, their accuracy significantly declines when facing a high noise ratio. 

Overall, \ours iteratively learns graph homophily to correct noise labels, ensures high-quality labels for training the GNN model,
and achieves high and stable accuracy under different noises.

\subsection{Ablation Study}
The ablation study is done to understand the main components of \ours. 
We first remove the whole label propagation step and call this variant \ours-nlp (\textbf{n}o \textbf{l}abel \textbf{p}ropagation step). 
Instead, we use self-training~\cite{self_training} that just picks the most confident noisy labels in the GNN and put them into the clean set to improve the performance.
Secondly, we replace the reconstructed graph with 
the similarity matrix and call this variant \ours-nz (\textbf{n}o \bm{$Z^{(l)*}$}).
Further, we do not propagate the noisy labels or the predicted labels in the label propagation, respectively, and call these two variants \ours-nnl (\textbf{n}o \textbf{n}oisy \textbf{l}abels) and \ours-npl (\textbf{n}o \textbf{p}redicted \textbf{l}abels).
We compare \ours with these four variants, and the results are presented in Figure~\ref{ablation}. Our findings show that \ours outperforms all the variants on the two datasets. Furthermore, the performance gap between \ours and \ours-nlp highlights the importance of label propagation for label correction. The noise labels significantly enhance the performance of label propagation, especially on heterophilous graphs.

\subsection{Performance without Clean Set}
\label{sec:no_clean}
For fairness,
we remove the clean label set and further study the model performance without clean set.
The results on Cora and Chameleon are given in Table~\ref{table:no_clean}.
For each dataset, we compare three noise rates from low to high under uniform/flip noise, which leads to a total of 12 comparisons.
From the table, 
we see that 
our method \ours\ achieves the best or the runner-up results 11 times, which shows that \ours\ are still effective even without clean set.
For other competitors,
although they perform well in some cases, they cannot 
provide consistently superior results.
For example,
under $40\%$ uniform noise,
Backward achieves the highest accuracy of 65.46 on Chameleon, but its accuracy is only 80.18 on Cora (the winner's is 84.32). 
All these results demonstrate the robustness of \ours\ against label noise. 

\begin{table}[]
\centering
\footnotesize
\caption{The classification results and the homophily of the graph with iterations.}
\vspace{-1em}
\label{ta:homo}
\begin{tabular}{@{}ccccccc@{}}
\toprule
Types              & Noise                 & Results    & 5 rounds & 10 rounds & 15 rounds & 20 rounds \\ \midrule
\multirow{4}{*}{Uniform} & \multirow{2}{*}{20\%} & Acc (\%)   & 66.67    & 69.08     & 70.83     & 71.59     \\
                         &                       & Edge Homo. & 0.54     & 0.57      & 0.58      & 0.58      \\
                         & \multirow{2}{*}{40\%} & Acc        & 65.79    & 65.57     & 67.59     & 68.23     \\
                         &                       & Edge Homo. & 0.53     & 0.55      & 0.57      & 0.57      \\ \midrule
\multirow{4}{*}{Flip}    & \multirow{2}{*}{20\%} & Acc (\%)   & 66.23    & 70.61     & 71.27     & 70.83     \\
                         &                       & Edge Homo. & 0.54     & 0.57      & 0.58      & 0.57      \\
                         & \multirow{2}{*}{40\%} & Acc (\%)   & 54.17    & 61.18     & 62.28     & 61.09     \\
                         &                       & Edge Homo. & 0.51     & 0.54      & 0.56      & 0.56      \\ \bottomrule
\end{tabular}
\end{table}

\subsection{The Homophily of Graphs with Iterations}
\label{sec:homo}
We further study classification accuracy and the homophily of the reconstructed graph with iterations. 
Specifically, 
we reported the results at five-round intervals in Table~\ref{ta:homo} using the Chameleon dataset. 
The initial homophily of Chameleon is 0.23. 
We observe that:
(1) Notably, $R^2LP$ significantly enhances the homophily of the reconstructed graph with iterations under different situations.
(2) As the number of learning rounds increases, both the reconstructed graph's homophily and accuracy gradually improve and converge.

\begin{table}[ht!]
\centering
\footnotesize
\caption{Average running time (s).}
\vspace{-1em}
\label{time}
\begin{tabular}{@{}cccc@{}}
\toprule
Methods              & Cora                                                   & Actor                                                  & arXiv-year                                             \\ \midrule
LPM                  &  41.96  &  158.32 &  512.37 \\
NRGNN                &  46.81  &  137.39 &  OOM    \\
RTGNN                & 143.16 &  283.76 &  OOM     \\
ERASE                &  132.50 & 180.51 & 455.13\\
\ours &  26.98  & 84.15  & 119.46 \\ \bottomrule
\end{tabular}
\end{table}

\vspace{-1em}
\subsection{Runtime Comparison}
\label{sec:runtime}

We compare the running time using the Cora, Actor and arXiv-year datasets in Table~\ref{time}, 
where all processes are executed on a single Tesla A100 GPU (80GB memory).
From the table, we see that \ours\ is much faster than other baselines.

\section{Conclusion}
In this paper, we address the problem of graph label noise under arbitrary graph heterophily.
We begin by empirically exploring the relationship between graph homophily and label noise, leading to two key observations: (1) graph homophily can indeed mitigate the effect of graph label noise on GNN performance, and (2) LP-based methods achieve great performance as graph homophily increases. These findings inspire us to combine LP with homophilous graph reconstruction to propose a 
effective algorithm \ours. It is a multi-round process that performs graph reconstruction with homophily, label propagation for noisy label refinement and high-confidence sample selection iteratively. 
Finally, we conducted extensive experiments and demonstrated the superior performance of \ours compared to other 10 state-of-the-art competitors on 10 benchmark datasets.
In the future, we plan to explore and improve \ours in more challenging scenarios, where the clean label set is unavailable, or input data, including node features and graph topology, also exhibits noise.

\begin{acks}
  This work is supported by National Natural Science Foundation of China No. 62202172, Shanghai Science and Technology Committee General Program No. 22ZR1419900 and
Singapore MOE AcRF Tier-2 Grant (T2EP20122-0003).
\end{acks}


\bibliographystyle{ACM-Reference-Format}
\balance
\bibliography{sample-base}

\clearpage

\appendix

\begin{table*}[h]
  \small
  \caption{Datasets statistics. Note that Edge Hom.~\cite{zhu2020beyond} is defined as the fraction of edges that connect nodes with the same label.}
  \vspace{-1em}
  \label{statistics_dataset}
  \begin{tabular}{c c c c c c c c c c c c}
    \toprule
     & \textbf{Cora} & \textbf{Citeseer}  & \textbf{Chameleon} &
     \textbf{Cornell} & \textbf{Wisconsin} &
     \textbf{Texas} & \textbf{Actor}& \textbf{Penn94} &\textbf{arXiv-year} &\textbf{snap-patents}\\
    \midrule
    Edge Hom. & 0.81 & 0.74  & 0.23 & 0.30 & 0.21 & 0.11 & 0.22 & 0.47 & 0.22 & 0.07 \\
    \#Nodes & 2,708 & 3,327 & 2,277 & 183 & 251 & 183 & 7,600 & 41,554 & 169,343 & 2,923,922\\
    \#Edges & 5,278 & 4,676 & 31,421 & 280 & 466 & 295 & 26,752 & 1,362,229 & 1,166,243 &13,975,788 \\
    \#Features & 1,433 & 3,703 & 2,325 & 1,703 & 1,703 & 1,703 & 931 & 5 &128 &269\\
    \#Classes & 7 & 6 & 5 & 5 & 5 & 5 & 5 & 2 & 5 & 5\\
  \bottomrule
\end{tabular}
\end{table*}


\section{Baselines}
\label{ap:baselines}


We compare \ours with 10 baselines, which can be categorized into three types. The first type involves (1) \textbf{GCN}~\cite{GCN} (2) \textbf{GloGNN}~\cite{li2022finding} and (3) \textbf{H$_{2}$GCN}~\cite{H2GCN}, three popular GNNs for homophilous and heterophilous graphs. They are directly trained on both clean and noisy training samples. 
In the second type, we modify two typical loss correction and sample selection methods and make them suitable for graph data. 
(4) \textbf{Co-teaching+}~\cite{yu2019does} is a multi-network training method to select clean samples. 
(4) \textbf{Backward}~\cite{patrini2017making} is to revise predictions and obtain unbiased loss on noisy training samples.
For fairness, we use GloGNN as a backbone GNN model for Co-teaching+ and Backward.
In the third type, we compare with five GNNs dealing with label noise, including 
(6) \textbf{NRGNN}~\cite{Nrgnn}, 
(7) \textbf{LPM}~\cite{LPM}, 
(8) \textbf{RTGNN}~\cite{RTGNN},
(9) \textbf{ERASE}~\cite{erase},
(10) \textbf{CGNN}~\cite{yuan2023learning}.
In detail, NRGNN assigns more edges between unlabeled nodes and (pseudo) labeled nodes against noise. LPM addresses the graph label noise by label propagation and meta-learning. RTGNN introduce self-reinforcement and consistency regularization as supplemental supervision to explicitly govern label noise. ERASE introduce a decoupled label propagation method  to learn representations with error tolerance by maximizing coding rate reduction. CGNN employ graph contrastive learning as a regularization term, which promotes two views of augmented nodes to have consistent representations which enhance the robustness of node representations to label noise.
Since the absence of publicly available code for CGNN, we validated \ours\ under the settings described in their paper and compared them with it. The results are presented in the appendix~\ref{sec:cgnn}.


\section{Dataset}
\label{ap:dataset_description}
We conduct experiments on 10 benchmark datasets, which include 2 homophilous graphs (Cora, Citeseer), 5 heterophilous graphs (Chameleon, Cornell, Wisconsin, Texas, Actor) and 3 large-scale heterophilous graphs (Penn94, arXiv-year, snap-patent).
For each dataset, we randomly
split nodes into 60\%, 20\%, and 20\%
for training, validation and testing, and measure the performance of all models on the test set.
The statistics and details of these datasets can be found in Table~\ref{statistics_dataset}, and we describe each dataset in detail below:

\noindent\textbf{\emph{Cora} and \emph{Citeseer}}
are two graphs commonly used as benchmarks.
In these datasets, each node represents a scientific paper and each edge represents a citation. Bag-of-words representations serve as feature vectors for each node. The label of nodes indicates the research field of the paper. 

\noindent\textbf{\emph{Texas}, \emph{Wisconsin} and \emph{Cornell}} are three heterophilous graphs representing links between web pages at the corresponding universities. In these datasets, each node represents a web page and each edge represents a hyperlink between nodes. Bag-of-words representations are also used as feature vectors for each node. The label is the web category. 

\noindent\textbf{\emph{Chameleon}} is a subgraph of web pages in Wikipedia, and the task is to classify nodes into five categories. 

\noindent\textbf{\emph{Actor}} is a heterophilous graph that represents actor co-occurrence in Wikipedia pages.
Each node represents the actor, and each edge is the co-occurrence in Wikipedia pages. Node features are derived from keywords included in the actor's Wikipedia page. The task is to categorize actors into five classes.

\noindent\textbf{\emph{Penn94}} is a subgraph extracted from Facebook whose nodes are students. Node features include major, second major/minor, dorm/house, year and high school. We take students’ genders as nodes’ labels. 

\noindent\textbf{\emph{arXiv-year} and \emph{snap-patents}} are two heterophilous citation networks. \emph{arXiv-year} is a directed subgraph of ogbn-arXiv, where nodes are arXiv papers and edges represent the citation relations. We construct node features by taking the averaged word2vec embedding vectors of tokens contained in both the title and abstract of papers. The task is to classify these papers into five labels that are constructed based on their posting year.
\emph{snap-patents} is a US patent network whose nodes are patents and edges are citation relations. Node features are constructed from patent metadata. Our goal is to classify the patents into five labels based on the time when they were granted.



\section{Experimental setup}
\label{sec:Experimental-setup}
We implemented \ours by PyTorch and conducted the experiments on 10 datasets with one A100 GPU. The model is optimized by Adam~\cite{adam}.
Considering that some methods are designed for scenarios without clean samples,
we finetune these baselines on the initial clean sets for fairness.
To fine-tune hyperparameters, we performed a grid search based on the result of the validation set.
Table~\ref{hyperparameters} summarizes the search space of hyper-parameters in \ours.
\begin{table}[H]
  \small
  \centering
  \caption{Grid search space.}
  \vspace{-1em}
  \label{hyperparameters}
  \begin{tabular}{c | c }
    \toprule
    \textbf{Hyper-parameter} & \textbf{Search space} \\
    \midrule
    lr & $ \{ 0.005, 0.01, 0.02, 0.03  \} $ \\
    dropout & $ [ 0.0, 0.9 ] $ \\
    weight decay &  \{1e-7, 5e-6, 1e-6, 5e-5, 1e-5, 5e-4\}  \\
    $\beta_{1}$  &$ \{0,1,10\} $ \\
    $\beta_{2}$  & $ \{0.1, 1, 10, 10^{2}, 10^{3}\} $ \\
    $\gamma$ & $ [0.0, 0.9] $ \\
    $\alpha _1$ & $ [ 0.0, 1.0 ] $ \\
    $\alpha _2$ & $ [ 0.0, 1.0  ] $ \\
    $\alpha _3$ & $ [ 0.0, 1.0  ] $ \\
    $\varepsilon$ & $ [ 0.1, 0.9] $ \\
  \bottomrule
\end{tabular}
\end{table}

\section{Additional Experimental Result}
\subsection{Graph Homophily Mitigates the Effect of Graph Label Noise}
\label{sec:homo_noise_uniform}

We conducted experiments on 4 existing methods, including two SOTA methods for graph label noise: LPM and NRGNN; and two classical methods: GCN and LP.
Figure~\ref{fg:homo_noise_uniform} shows the performance of all the methods under different levels of uniform label noise.

\begin{figure*}[]
    \centering
    \includegraphics[scale=0.4]{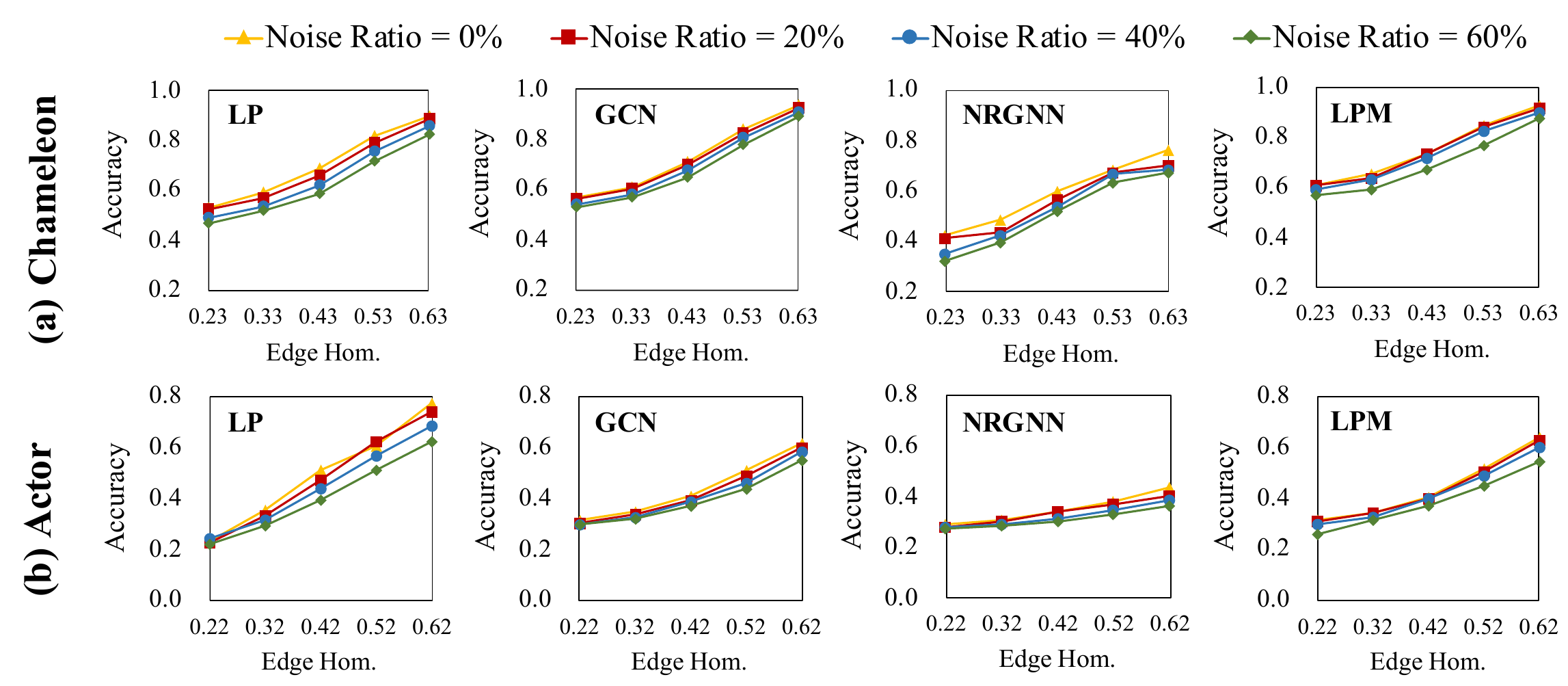}
    \vspace{-1em}
    \caption{The impact of edge homophily on graph label noise across various methods (uniform noise)}
    \label{fg:homo_noise_uniform}
 \end{figure*}

\subsection{Performance Comparison with CGNN}
\label{sec:cgnn}
We provide the mean accuracy of 5 runs for each baseline. In our experiments, the training label rate is set to 0.01 and the percent
of noisy labels is fixed at 20\%. 
Table~\ref{cgnn} presents the result on Coauthor CS, Amazon Photo and Amazon Computers.
From the tables, we can see \ours consistently outperforms the
competitive methods across the uniform and flip noise settings in the majority of cases.

\begin{table}[H]
\footnotesize
\caption{The classification accuracy (\%) compared with CGNN.}
\vspace{-1em}
\begin{tabular}{ccccc}
\toprule
\textbf{Methods}               & \textbf{Noise}        & \textbf{Coauthor CS} & \textbf{Amazon Photo} & \textbf{ \makecell[c]{Amazon \\ Computers}} \\ 
\hline
\multirow{2}{*}{CGNN} & Uniform-20\% & 84.1        & 85.3         & 81.8             \\
                      & Flip-20\%    & 81.0        & 85.1         & 81.6             \\
\hline
\multirow{2}{*}{\ours}  & Uniform-20\% & 85.3        & 85.8         & 83.6             \\
                      & Flip-20\%    & 83.1        & 84.6         & 82.7 \\
\bottomrule
\end{tabular}
\label{cgnn}
\end{table}

 \begin{figure}[H]
  \centering 
  \includegraphics[width=0.9\linewidth]{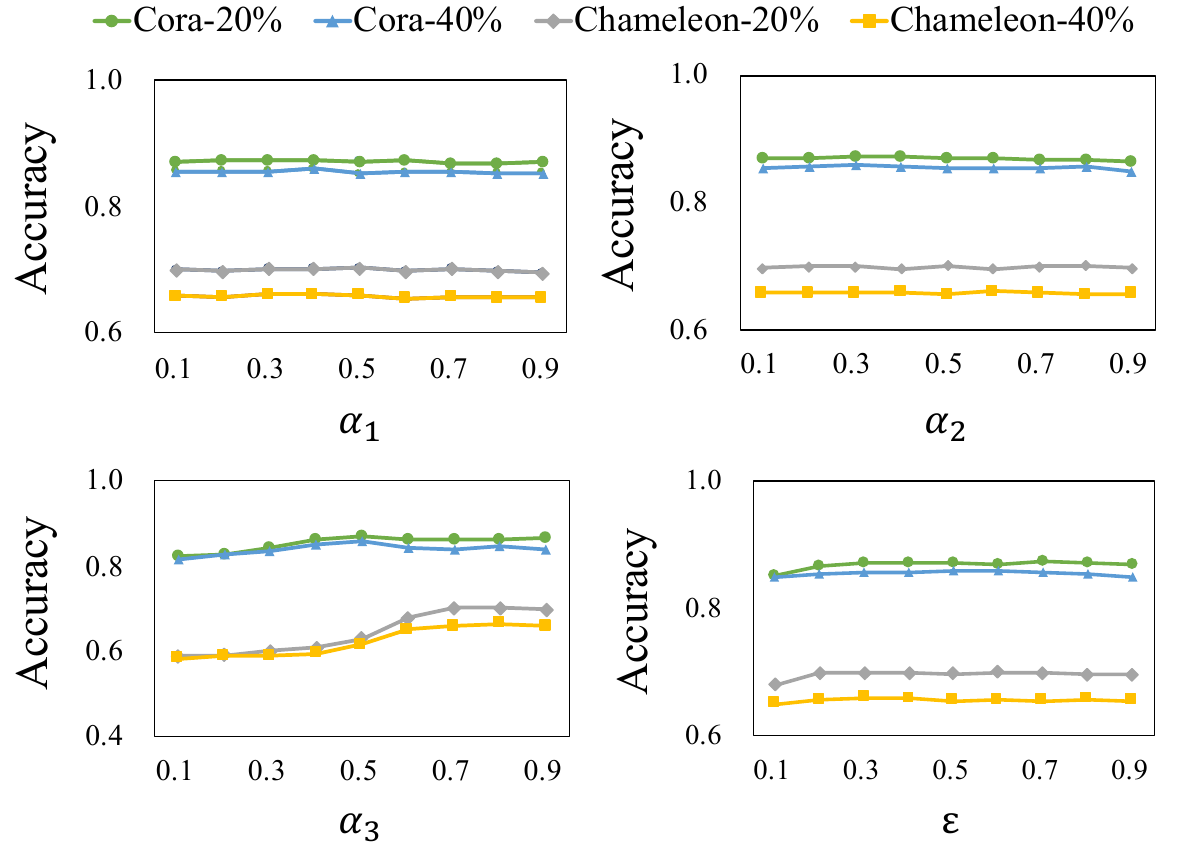}
  \vspace{-1em}
  \caption{Hyper-parameter sensitivity analysis   }
  \label{paras}
\end{figure}

\subsection{Hyper-parameter Sensitivity Analysis}
\label{sec:Sensitivity_Analysis}
Here we study the sensitivity of four hyper-parameters:
the weight $\alpha _{1}$, $\alpha _{2}$, $\alpha _{3}$ and the sample selection ratio $\varepsilon$. 
The results are shown in Figure~\ref{paras}.
For all the hyper-parameters, \ours consistently performs well across a wide range of values, which demonstrates that these hyper-parameters do not significantly impact the performance of \ours.


\subsection{Analysis of Selection Strategies}
\label{sec:Selection_Strategies}
We compare the performance of \ours by using different selection functions: 
(1) \textit{Threshold}~\cite{yu2019does} chooses the noisy labels whose confidence is greater than a certain threshold;
(2) \textit{Absolute number} selects an absolute number of noisy labels with high confidence based on the current number of samples;
(3) \textit{Statistical selection}~\cite{patel2023adaptive} uses statistics to derive the threshold for sample selection;
(4) \ours selects samples with high confidence based on a certain proportion relative to the current noisy nodes.
Figure~\ref{select} shows the results on Cora and Chameleon datasets. 
From the figures, we observe that the accuracy drops dramatically as the noise ratio increases when using \textit{Threshold} and \textit{Absolute number}. The reason is that they choose many noisy labels into the clean set. \textit{Statistical selection} accurately selects the node labels, but too stringent selection conditions limit the number of selected labels. 
Compared with these functions, \ours selects the labels with respect to the suitable quantity and high quality of remaining samples. 
\begin{figure*}[ht!]
  \centering  \includegraphics[width=0.83\linewidth]{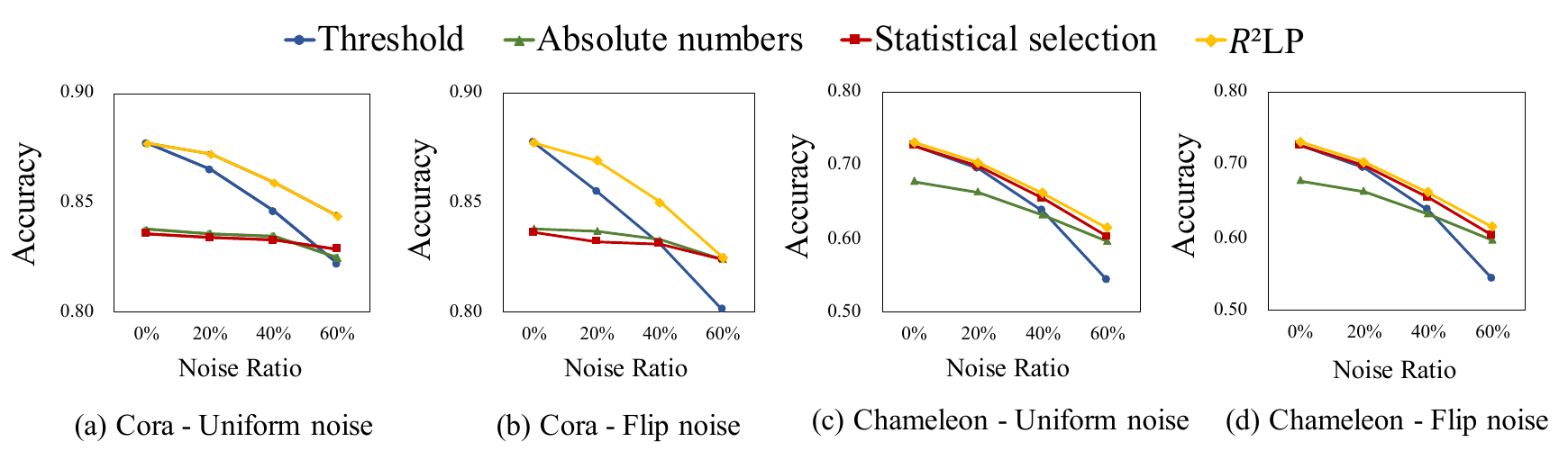}
  \vspace{-1.2em}
  \caption{The effect of different selection strategies.}
  \label{select}
\end{figure*}

\begin{table*}
\setlength{\tabcolsep}{2.5pt}
  \centering
  \caption{The classification accuracy (\%) over the methods on 10 datasets with uniform noise ranging from 0\% to 60\%. We highlight the best score on each dataset in bold. OOM denotes the out-of-memory error. 
  }
\vspace{-1em}
  \label{uniform}
  \resizebox{\linewidth}{!}
  {
  \begin{tabular}{ccccccccccccc}
    \toprule
    \multicolumn{1}{c|} {\textbf{Types}} & \textbf{Methods} &  \multicolumn{1}{c|}{\textbf{Noise}} &\textbf{Cora} & \textbf{Citeseer} & 
    \textbf{Chameleon} & \textbf{Cornell} & \textbf{Wisconsin} &
    \textbf{Texas} & \textbf{Actor} &\textbf{Penn94} & \textbf{arXiv-year} & \textbf{snap-patents} \\ 
    \hline
    \multicolumn{1}{c|}{\multirow{12}{*}{
    GNNs
    }} & \multirow{4}{*}{GCN} &\multicolumn{1}{c|}{0\%}& 87.75 &76.81  &  56.95 & 61.89 &63.92  & 62.16 & 31.27 & 80.76 & 44.67  & 53.59 \\
    \multicolumn{1}{c|}{} &  & \multicolumn{1}{c|}{20\%} & 86.54 &75.66  &54.29  & 61.08 & 63.52 & 62.43 & 30.27 & 80.29 & 39.59 & 50.89 \\
    \multicolumn{1}{c|}{}  &  & \multicolumn{1}{c|}{40\%} & 85.74 & 75.05 & 48.90 & 61.08 & 60.19 & 62.43 & 30.08 & 78.21 & 38.17 & 48.39\\
    \multicolumn{1}{c|}{} &  & \multicolumn{1}{c|}{60\%} & 83.67 & 73.44 & 33.13 &60.54  & 60.78 & 61.62 & 29.79 & 76.74 & 36.88 & 46.81\\ 
    \cline{2-13} 
    
    \multicolumn{1}{c|}{\multirow{12}{*}{}}  &\multirow{4}{*}{GloGNN} 
    &\multicolumn{1}{c|}{0\%}&  84.58 & 76.09 &  69.83 & 83.51 & 88.43 & 84.05 & 38.55 &85.66 &52.74 &62.33\\
    \multicolumn{1}{c|}{} &  & \multicolumn{1}{c|}{20\%} & 82.53 & 74.64 & 69.38 & 80.81 & 84.31 & 80.00 & 37.86 &81.85 &50.47 &58.96\\
    \multicolumn{1}{c|}{}  &  & \multicolumn{1}{c|}{40\%} & 79.92 & 72.62 & 65.78 & 74.86 & 74.50 & 68.37 & 37.33 &77.44 &47.77 &54.97\\
    \multicolumn{1}{c|}{} &  & \multicolumn{1}{c|}{60\%} &  75.70 & 69.90 & 61.55 & 69.45 & 66.27 & 61.89 & 36.06 &71.59 &44.82 &50.12\\ 
    \cline{2-13} 
    
    \multicolumn{1}{c|}{\multirow{12}{*}{}}  &\multirow{4}{*}{H$_{2}$GCN} 
    &\multicolumn{1}{c|}{0\%}& 87.87 & 77.11 & 60.00 & 82.59 & 85.65 & 86.26 & 35.70 &79.02 &48.72 &57.21 \\
    \multicolumn{1}{c|}{} &  & \multicolumn{1}{c|}{20\%} &  85.34 & 74.51  & 55.31 & 74.27 & 75.09 & 77.27  & 31.25 &76.22 &45.60 &54.68\\
    \multicolumn{1}{c|}{}  &  & \multicolumn{1}{c|}{40\%} & 82.46  & 73.17 &  51.79 & 67.02 & 74.50 & 69.45 & 31.95 &73.73 &40.13 &50.18\\
    \multicolumn{1}{c|}{} &  & \multicolumn{1}{c|}{60\%} &  80.94 & 70.87  & 46.60& 57.56 & 69.01 & 62.16  & 31.25 &71.01 &33.65 &41.36\\
    \hline
    
   \multicolumn{1}{c|}{\multirow{8}{*}{\makecell{Methods\\
   for \\Label \\Noise}}} & \multirow{4}{*}{Co-teaching+} &\multicolumn{1}{c|}{0\%} & 85.76  & 76.32 & 72.52  & 70.54  & 72.68  &  67.21 & 34.56 &86.90 &OOM &OOM  \\
    \multicolumn{1}{c|}{} &  & \multicolumn{1}{c|}{20\%} & 80.46  & 73.32   & 67.30 & 69.91 & 60.19 & 64.56  & 33.62 &82.73 &OOM &OOM\\
    \multicolumn{1}{c|}{}  &  & \multicolumn{1}{c|}{40\%} & 70.00  & 63.00  &58.81  &54.05 & 54.05 & 60.70  & 30.27  &73.86 &OOM &OOM\\
    \multicolumn{1}{c|}{} &  & \multicolumn{1}{c|}{60\%} & 59.17   & 52.29  & 51.90 & 49.18 & 43.72 & 43.24 & 27.06  &52.03 &OOM &OOM\\
    \cline{2-13}

    \multicolumn{1}{c|}{\multirow{12}{*}{}}  &\multirow{4}{*}{Backward} 
    &\multicolumn{1}{c|}{0\%}& 84.88  & 77.06  & 71.00  & 72.70  & 81.37  & 84.31  & 23.90  &\textbf{87.00} &46.51 &58.54\\
    \multicolumn{1}{c|}{} &  & \multicolumn{1}{c|}{20\%} &  83.35& 75.52 &69.34  &68.91 & 75.68 & 76.07  & 22.13  &84.34 &43.85 &56.41\\
    \multicolumn{1}{c|}{}  &  & \multicolumn{1}{c|}{40\%} & 80.18  & 69.89 & 65.76 & 68.64 & 71.72 & 70.78  & 22.98 &80.90 &33.16 &56.37\\
    \multicolumn{1}{c|}{} &  & \multicolumn{1}{c|}{60\%} & 76.16   & 67.20 & \textbf{61.93} & 58.37 &66.66  &67.19  & 21.28 &76.71 &27.80 &52.25 \\
    \hline

    \multicolumn{1}{c|}{\multirow{12}{*}{\makecell{Methods \\for\\ Graph\\ Label\\ Noise}}} & \multirow{4}{*}{NRGNN} &\multicolumn{1}{c|}{0\%} & 84.14  &73.93 & 42.64 & 69.91  &  70.80 &  72.64 &  28.95 &68.31 &OOM &OOM\\
    \multicolumn{1}{c|}{} &  & \multicolumn{1}{c|}{20\%} & 82.73  & 72.94& 41.17 & 68.84 & 68.62 &  71.05& 28.24 &67.02 &OOM &OOM\\
    \multicolumn{1}{c|}{}  &  & \multicolumn{1}{c|}{40\%} &  80.66 & 63.88   & 35.26 & 61.37 & 58.17 & 61.75  & 28.15 &65.89 &OOM &OOM\\
    \multicolumn{1}{c|}{} &  & \multicolumn{1}{c|}{60\%} & 73.67   & 62.95 & 31.97 & 58.19 & 50.39& 57.56& 27.24 &53.60 &OOM &OOM\\
    \cline{2-13}
    
    \multicolumn{1}{c|}{\multirow{12}{*}{}}  &\multirow{4}{*}{LPM} 
    &\multicolumn{1}{c|}{0\%}& \textbf{89.74} & \textbf{78.77}  & 60.72  & 63.87 & 73.72 & 69.46  &  31.43 &76.10 & 44.32& 56.76\\
    \multicolumn{1}{c|}{} &  & \multicolumn{1}{c|}{20\%} & 87.06 & \textbf{78.24}  &60.57  & 62.70 & 72.15  & 69.19  & 30.95 &75.62 &42.15 &54.13\\
    \multicolumn{1}{c|}{}  &  & \multicolumn{1}{c|}{40\%} & 84.81  & 74.50  & 59.38&62.43 & 69.21 &68.11 & 29.62 &73.50 & 39.44&50.27\\
    \multicolumn{1}{c|}{} &  & \multicolumn{1}{c|}{60\%} & 81.42  & 70.97  & 57.10& 62.16 & 67.84 &63.78  & 25.71 &69.98 &35.41 &44.36\\
    \cline{2-13}

    \multicolumn{1}{c|}{\multirow{12}{*}{}} & \multirow{4}{*}{RTGNN} &\multicolumn{1}{c|}{0\%} & 86.08&	76.27&	45.88&	62.83&	67.64&	71.43&	29.55&	72.65&	OOM&	OOM\\
    \multicolumn{1}{c|}{} &  & \multicolumn{1}{c|}{20\%} & 85.49&	75.03&	44.39&	59.45&	63.75&	68.91&	28.94&	72.01&	OOM&	OOM\\
    \multicolumn{1}{c|}{}  &  & \multicolumn{1}{c|}{40\%} & 84.33&	73.27&	39.03&	57.43&	61.27&	62.16&	28.70&	65.04&	OOM&	OOM\\
    \multicolumn{1}{c|}{} &  & \multicolumn{1}{c|}{60\%} & 83.33&	71.65&	32.23&	55.40&	57.84&	59.45&	28.88&	62.32&	OOM&	OOM\\
    \cline{2-13}

    \multicolumn{1}{c|}{\multirow{12}{*}{}} & \multirow{4}{*}{ERASE} &\multicolumn{1}{c|}{0\%} & 87.32&	77.14&	50.15&	62.62&	65.52&	72.91&	30.16&	75.36&	43.27&	56.50\\
    \multicolumn{1}{c|}{} &  & \multicolumn{1}{c|}{20\%} & 86.83&	75.43&	45.76&	58.26&	64.26&	70.23&	30.01&	72.89&	42.71&	55.08\\
    \multicolumn{1}{c|}{}  &  & \multicolumn{1}{c|}{40\%} & 84.92&	74.16&	40.21&	57.19&	63.51&	67.50&	29.88&	66.53&	38.76&	48.57\\
    \multicolumn{1}{c|}{} &  & \multicolumn{1}{c|}{60\%} & 83.67&	72.07&	33.51&	53.58&	59.06&	62.11&	28.43&	63.63&	35.22&	40.29\\
    \cline{2-13}
    
    \multicolumn{1}{c|}{\multirow{12}{*}{}}  &\multirow{4}{*}{\ours} 
    &\multicolumn{1}{c|}{0\%}&87.69  & 78.09 & \textbf{73.05} &  \textbf{87.84}& \textbf{88.82} & \textbf{87.30} &  \textbf{38.59} &86.78 &\textbf{53.80} & \textbf{63.04}\\
    \multicolumn{1}{c|}{} &  & \multicolumn{1}{c|}{20\%} & \textbf{87.23} & 76.54 &\textbf{70.28} & \textbf{87.03} & \textbf{88.43}& \textbf{86.49} & \textbf{38.46}&\textbf{84.72} &\textbf{52.18} &\textbf{59.86}\\
    \multicolumn{1}{c|}{}  &  & \multicolumn{1}{c|}{40\%} & \textbf{85.90} & \textbf{75.16} &\textbf{66.23} &\textbf{80.00}&\textbf{83.92} &  \textbf{78.37}&\textbf{37.51} &\textbf{83.03} &\textbf{50.42} &\textbf{57.33 }\\
    \multicolumn{1}{c|}{} &  & \multicolumn{1}{c|}{60\%} & \textbf{84.37}  & \textbf{73.91} & 61.40& \textbf{74.32} &\textbf{81.18}& \textbf{73.24}& \textbf{36.82}&\textbf{81.48} &\textbf{48.01} &\textbf{55.29}\\
    \bottomrule
  \end{tabular}
  }
\end{table*}

\clearpage
\newpage
\onecolumn

\section{Proofs}
\label{ap:proof}

\subsection{Proof of Theorem \ref{thm:denoise}}\label{app:prof_denoise}

Remember we obtain the propagated label of the node $i$ by
$
\hat{Y}_{i} = (1-\alpha) \tilde{Y}_{i} + \frac{\alpha }{d} \sum_{j \in [d]} \tilde{Y}_{j}. 
$
The first part is its own noisy label $\tilde{Y}_{i}$, and 
the second part is the label aggregation from its $d$ neighbors' noisy labels $\tilde{Y}_{j}$. We first analyze the distribution of $Y^{\text{neighbor}} = \sum_{j \in [d]} \tilde{Y}_{j}$. 

\begin{theorem}
Assume that the true label of node $i$ is $Y_i=0$. Then the distribution of $Y^{\text{neighbor}}= \sum_{j \in [d]} \tilde{Y}_{j}$ follows a conditional binomial given by $Y^{\text{neighbor}} \sim B(d, pe+(1-p)(1-e))$. Similarly, when $Y_i=1$, the distribution of $Y^{\text{neighbor}}$ is also a conditional binomial $Y^{\text{neighbor}} \sim B(d, p(1-e)+(1-p)e)$.
\end{theorem}

In both cases, the parameter $d$ represents the number of neighbors, $p$ means the probability of the neighbor having the same label as node $i$, and $e$ represents the probability of an error in the labeling process, where $\Pr(\tilde{Y}_i=1|Y_i=0)=e$ and  $\Pr(\tilde{Y}_i=0|Y_i=1)=e$.

\begin{proof}
Let $N_s$ be a random variable representing the number of neighbors that have the same true labels with the node $i$, where $0 \leq N_s \leq d$. The probability that $Y^{\text{neighbor}} = m$ can be expressed as the sum of conditional probabilities:
$$\Pr(Y^{\text{neighbor}} = m)  = \sum_{i = 0}^{d} \Pr(N_s=i) \Pr[Y^{\text{neighbor}} = m \mid N_s=i] $$
The case $Y^{\text{neighbor}} = m$ comes from two parts. The first part is that $Y^{\text{neighbor}} = j$ originates from the neighbor nodes that share the same true labels as node $i$. The second part is that $Y^{\text{neighbor}} = m-j$ comes from the neighbor nodes that have different true labels. Hence, the equation can be expressed as follows:
\begin{equation*}\small
\begin{aligned} 
&= \sum_{i = 0}^{d} \Pr(N_s=i) \left[ \sum_{j=0}^{\text{min}(i,m)} \Pr[Y^{\text{neighbor}} = j \text{ in } N_s , Y^{\text{neighbor}} = m-j \text{ not in } N_{s} ] \right] \\
&=  \sum_{i = 0}^{d}  \binom{d}{i} p^i  (1-p)^{d-i} \left[  \sum_{j=0}^{\text{min}(i,m)}  \binom{i}{j}  e^j (1-e)^{(i-j)} \binom{d-i}{m-j}    (1-e)^{(m-j)} e^{(d-i-m+j)} \right] \\
&=  \sum_{i = 0}^{d}  \binom{d}{i} p^i  (1-p)^{d-i} \left[  \sum_{j=0}^{\text{min}(i,m)}  \binom{i}{j} \binom{d-i}{m-j}    (1-e)^{(m+i-2j)} e^{(d-i-m+2j)} \right] \\
&=  \sum_{i = 0}^{d}  \sum_{j=0}^{\text{min}(i,m)} \left[  \binom{d}{i} \binom{i}{j} \binom{d-i}{m-j}  p^i  (1-p)^{d-i}   (1-e)^{(m+i-2j)} e^{(d-i-m+2j)} \right] \\
&=  \sum_{i = 0}^{d}  \sum_{j=0}^{\text{min}(i,m)} \left[  \frac{d!}{j!(i-j)!(m-j)!(d-i-m+j)!}  p^i  (1-p)^{d-i}   (1-e)^{(m+i-2j)} e^{(d-i-m+2j)} \right]
\end{aligned}
\end{equation*}
Let $a=i-j$ and $b=j$, we have $i=a+b$ and $j=b$.
Because of $0 \leq i \leq d$, $0 \leq j \leq m$, $j\leq i$ and $i-j \leq d-m$, we can obtain $0 \leq a \leq d-m$ and $0 \leq b \leq m$.
The above equation is rewritten as
\begin{equation*}\small
\begin{aligned} 
&= \sum_{a=0}^{d-m} \sum_{b = 0}^{m}   \left[  \frac{d!}{b!a!(m-b)!(d-m-a)!}  p^{a+b}  (1-p)^{d-a-b}   (1-e)^{(m+a-b)} e^{(d-m-a+b)} \right]\\
&= \sum_{a=0}^{d-m} \sum_{b = 0}^{m}  \left[   \binom{d}{m} \binom{m}{b} \binom{d-m}{a}  p^{a+b}  (1-p)^{d-a-b}   (1-e)^{(m+a-b)} e^{(d-m-a+b)} \right] \\
&= \sum_{a=0}^{d-m} \sum_{b = 0}^{m}  \left[   \binom{d}{m} \binom{m}{b} \binom{d-m}{a}  \left(pe\right)^b \left((1-p)(1-e)\right)^{m-b} \left(p(1-e)\right)^a \left((1-p)e\right)^{d-m-a}  \right] \\
&= \sum_{a=0}^{d-m} \sum_{b = 0}^{m}  \left[   \binom{d}{m} \binom{m}{b} \binom{d-m}{a}  \left(pe\right)^b \left((1-p)(1-e)\right)^{m-b} \left(p(1-e)\right)^a \left((1-p)e\right)^{d-m-a}  \right] \\
&= \binom{d}{m}    \left[  \sum_{b = 0}^{m}\binom{m}{b}  \left(pe\right)^b \left((1-p)(1-e)\right)^{m-b}  \right]  \left[  \sum_{a=0}^{d-m}\binom{d-m}{a}\left(p(1-e)\right)^a \left((1-p)e\right)^{d-m-a} \right] \\ 
&= \binom{d}{m}    \left[  pe+(1-p)(1-e)\right]^m  \left[ p(1-e)+(1-p)e \right]^{d-m} \\ 
\end{aligned}
\end{equation*}
Therefore, we prove that $\Pr(Y^{\text{neighbor}} = m) = \binom{d}{m}    \left[  pe+(1-p)(1-e)\right]^m  \left[ p(1-e)+(1-p)e \right]^{d-m}$ which represents that $Y^{\text{neighbor}} \sim B(d, pe+(1-p)(1-e))$. Similarly, we can also prove that $Y^{\text{neighbor}} \sim B(d, p(1-e)+(1-p)e)$ when $Y_{i}=1$.
\end{proof}


\begin{theorem} 
Given the true label $Y_{i}=0$, the expectation of $\hat{Y}_{i}$ can be calculated as $\mathbb{E}(\hat{Y}_{i}|Y_{i}=0) = (1-\alpha)e+\alpha[pe+(1-p)(1-e)]$.
Similarly, when $Y_{i}=1$, the expectation of $\hat{Y}_{i}$ can be calculated as $\mathbb{E}(\hat{Y}_{i}|Y_{i}=1) = (1-\alpha)(1-e)+\alpha[p(1-e)+(1-p)e]$.
Furthermore, the variance of $\hat{Y}_{i}$ can be represented as $\text{Var}(\hat{Y}_{i}) = (1-\alpha)^2e(1-e) + \frac{\alpha^2 }{d}[pe+(1-p)(1-e)][p(1-e)+(1-p)e]$.
\end{theorem}

\begin{proof}
When $Y_{i}=0$, we have  $\Pr(\tilde{Y}_i=1|Y_i=0)=e$  and  $\Pr(\tilde{Y}_i=0|Y_i=0)=1-e$. Then the first part $\mathbb{E}[(1-\alpha) \tilde{Y}_{i} | Y_i=0 ]= (1-\alpha)e$. Since we have proven that $Y^{\text{neighbor}} \sim B(d, pe+(1-p)(1-e))$, the second part $\mathbb{E}[\frac{\alpha }{d} \sum_{j \in [d]} \tilde{Y}_{j} | Y_i=0 ]= \frac{\alpha }{d}\cdot d \cdot [pe+(1-p)(1-e)] = \alpha[pe+(1-p)(1-e)]$. By summing up these two parts, we obtain $\mathbb{E}(\hat{Y}_{i}|Y_{i}=0) = (1-\alpha)e+\alpha[pe+(1-p)(1-e)]$.

Similarly, when $Y_{i}=1$, we have $\Pr(\tilde{Y}_i=0|Y_i=1)=e$  and  $\Pr(\tilde{Y}_i=1|Y_i=1)=1-e$. Therefore, 
$\mathbb{E}[ \tilde{Y}_i | Y_i=1 ]= (1-\alpha)(1-e)+ \frac{\alpha }{d}\cdot d \cdot [p(1-e)+(1-p)e] 
= (1-\alpha)(1-e)+ \alpha \cdot [p(1-e)+(1-p)e].$

The variance of $\hat{Y}_{i}$ can be written as:
\begin{equation*}\small
\begin{aligned} 
\text{Var}(\hat{Y}_{i}) &= (1-\alpha)^2 \text{Var}(\tilde{Y}_{i}) + \frac{\alpha^2 }{d^2} \text{Var}(\sum_{j \in [d]} \tilde{Y}_{j})\\
&= (1-\alpha)^2e(1-e)+\frac{\alpha^2 }{d^2} \cdot d \cdot [pe+(1-p)(1-e)][p(1-e)+(1-p)e] \\
&= (1-\alpha)^2e(1-e)+\frac{\alpha^2 }{d}  [pe+(1-p)(1-e)][p(1-e)+(1-p)e]
\end{aligned}
\end{equation*}
\end{proof}

\begin{theorem*}(Label Propagation and Denoising) Suppose the label noise is generated by $\bm{T}$ and the label propagation follows Equation~\ref{eq:lp}. For a specific node $i$, we further assume the node has $d$ neighbors, and its neighbor nodes have the probability $p$ to have the same true label with node $i$, i.e., $\mathbb{P}[Y_i = Y_j] = p$. After one-round label propagation, the gap between the propagated label and the true label is expressed as
\begin{align}\label{eq:denoising} \small
    \mathbb{E}(Y-\hat{Y})^2 = \mathbb{P}(Y=0)\mathbb{E}(\hat{Y}|Y=0)^2 + \mathbb{P}(Y=1)\left(\mathbb{E}(\hat{Y}| Y =1) - 1 \right)^2+ \text{Var}(\hat{Y}),
\end{align}
where 
\begin{subequations} \small
\begin{align}
    &\mathbb{E}(\hat{Y}_{i}|Y_{i}=0)= (1-\alpha)e+\alpha[pe+(1-p)(1-e)],\label{eq:e0} \\
    &\mathbb{E}(\hat{Y}_{i}|Y_{i}=1)= (1-\alpha)(1-e)+\alpha[p(1-e)+(1-p)e], \label{eq:e1} \\
    &\text{Var}(\hat{Y}_{i})= (1-\alpha)^2e(1-e)+\frac{\alpha^2 }{d} [pe+(1-p)(1-e)] [1-pe-(1-p)(1-e)]. \label{eq:var}
\end{align}
\end{subequations}
\end{theorem*}

\begin{proof}
\begin{equation*}\small
\begin{aligned} 
\mathbb{E}(\hat{Y}-Y)^2 &= \mathbb{P}(Y=0) \mathbb{E}\left((\hat{Y}-Y)^2 | Y=0\right) + \mathbb{P}(Y=1) \mathbb{E}\left((\hat{Y}-Y)^2 | Y=1\right) \\
&= \mathbb{P}(Y=0) \mathbb{E}\left(\hat{Y}^2 | Y=0\right) + \mathbb{P}(Y=1) \mathbb{E}\left((\hat{Y}-1)^2 | Y=1\right) \\
&= \mathbb{P}(Y=0) \left[ \mathbb{E}(\hat{Y} | Y=0)^2+ \text{Var}(\hat{Y} | Y=0) \right]+ \mathbb{P}(Y=1) \left[ \mathbb{E}(\hat{Y}-1 | Y=1)^2+ \text{Var}(\hat{Y} | Y=1) \right] \\
&= \mathbb{P}(Y=0)\mathbb{E}(\hat{Y}|Y=0)^2 + \mathbb{P}(Y=1)\left(\mathbb{E}(\hat{Y}| Y =1) - 1 \right)^2+ \text{Var}(\hat{Y}),
\end{aligned}
\end{equation*}
The detailed expectation and variance of $\hat{Y}_{i}$ have been proved in the above.
\end{proof}

\subsection{Proof of Theorem \ref{thm:gen}}\label{app:prof_gen}

\begin{theorem*} (Generalization Error and Oracle Inequality)
    Denote the node feature $X$ sampled from the distribution $D$, the graph topology as $A$, the set of training nodes as $S_n$, the neural network as $f(\cdot, \cdot)$, and $\hat{f} = \inf_f \sum_{X\in S_n} (\hat{Y} - f(X, A))$. Assume that the propagated label concentrated on its mean, i.e., with probability at least $\frac{\delta}{2}$, $||\hat{Y} - Y| - \mathbb{E}|\hat{Y} - Y|| \leq \epsilon_1$. We further assume the generalization error is bounded with respect to the propagated labels, i.e., with probability at least $\frac{\delta}{2}$,
    \begin{align} \small
        \mathbb{E}_{X \sim D} |\hat{Y} - \hat{f}(X,A)| - \inf_f \mathbb{E}_{X \sim D} |\hat{Y} - f(X,A)| \leq \epsilon_2.
    \end{align}
    Then we obtain the generalization error bound for training with noisy labels and test on the clean labels, i.e., with probability at least $\delta$, the generalization error trained with propagated labels is given by
    \begin{align} \small
        \mathbb{E}_{X \sim D} |Y - \hat{f}(X,A)| \leq  \inf_f \mathbb{E}_{X \sim D} |Y - f(X,A)| + \epsilon_2 + 2(\mathbb{E}|\hat{Y} - Y| + \epsilon_1). 
    \end{align}
\end{theorem*}

\begin{proof}
We denote $f^* = \inf_f \mathbb{E}_{X \sim D} (Y - f(X,A))$, $f^{\prime} = \inf_f \mathbb{E}_{X \sim D} (\hat{Y} - f(X,A))$. Then
\begin{align}
    \mathbb{E}_{X \sim D} |Y - \hat{f}(X,A)| &\leq \mathbb{E}_{X \sim D} |\hat{Y} - \hat{f}(X, A)| + |Y - \hat{Y}| \\
    &\overset{(a)}{\leq} \mathbb{E}_{X \sim D} |\hat{Y} - f^{\prime}(X, A)| + \epsilon_2 + |Y - \hat{Y}|\\
    &\overset{(b)}{\leq} \mathbb{E}_{X \sim D} |\hat{Y} - f^*(X, A)| + \epsilon_2 + |Y - \hat{Y}|\\
    &\leq \mathbb{E}_{X \sim D} |Y - f^*(X,A)| + |Y - \hat{Y}| + \epsilon_2 + |Y - \hat{Y}|\\
    &= \mathbb{E}_{X \sim D} |Y - f^*(X,A)| +  \epsilon_2 + 2|Y - \hat{Y}|\\
    &\overset{(c)}{\leq} \mathbb{E}_{X \sim D} |Y - f^*(X,A)| + \epsilon_2 + 2(\mathbb{E}|Y - \hat{Y}| + \epsilon_1)
\end{align}
where (a) follows the generalization bound with respect to propagated labels, (b) follows $\inf_f \mathbb{E}_{X \sim D} |\hat{Y} - f(X, A)| \leq \mathbb{E}_{X \sim D} |\hat{Y} - f^*(X, A)|$, and (c) follows the concentration of $|Y - \hat{Y}|$. The probability $\delta$ is obtained by union bound.     
\end{proof}

\end{document}